\documentclass{article}

\usepackage[main, final]{neurips_2026}

\usepackage[utf8]{inputenc} 
\usepackage[T1]{fontenc}    
\usepackage{hyperref}       
\usepackage{url}            
\usepackage{booktabs}       
\usepackage{amsfonts}       
\usepackage{nicefrac}       
\usepackage{microtype}      
\usepackage{xcolor}         

\usepackage{algorithm}
\usepackage{algpseudocode}

\usepackage{amsmath}
\usepackage{amssymb}
\usepackage{amsfonts}
\usepackage{mathtools}
\usepackage{amsthm}
\usepackage{mathrsfs}%
\usepackage{array}
\usepackage{comment}
\usepackage{enumitem}

\usepackage[capitalize,noabbrev]{cleveref}

\theoremstyle{plain}
\newtheorem{theorem}{Theorem}[section]
\newtheorem{proposition}[theorem]{Proposition}
\newtheorem{lemma}[theorem]{Lemma}

\usepackage{multirow}%
\usepackage{makecell}
\usepackage{colortbl}
\usepackage{caption}
\definecolor{tabgray}{gray}{0.90}
\definecolor{grayblue}{RGB}{96, 125, 139}
\usepackage{bbding}
\usepackage{utfsym}
\usepackage{pifont}

\usepackage{wrapfig}

\usepackage{xcolor}

\usepackage{xspace}
\newcommand{\ours}{LOFT\xspace}

\newcommand{\fgdata}{\mathcal{D}_{\rm fg}\xspace}
\newcommand{\rmdata}{\mathcal{D}_{\rm rm}\xspace}
\newcommand{\fgpdata}{\mathcal{D}_{\rm fgp}\xspace}

\newcommand{\exactf}{f_{\rm exact}\xspace}
\newcommand{\pref}{f_{\rm pre}\xspace}
\newcommand{\exactg}{g_{\rm exact}\xspace}
\newcommand{\preg}{g_{\rm pre}\xspace}
\newcommand{\exacth}{h_{\rm exact}\xspace}
\newcommand{\preh}{h_{\rm pre}\xspace}

\newcommand{\exactsigmarm}{{\bf\Sigma}_{\rm rm}^{\rm exact}\xspace}
\newcommand{\presigmarm}{{\bf\Sigma}_{\rm rm}^{\rm pre}\xspace}
\newcommand{\exactsigmafg}{{\bf\Sigma}_{\rm fg}^{\rm exact}\xspace}
\newcommand{\presigmafg}{{\bf\Sigma}_{\rm fg}^{\rm pre}\xspace}
\newcommand{\presigmafgp}{{\bf\Sigma}_{\rm fgp}^{\rm pre}\xspace}

\newcommand{\featrm}{{\bf Z}_{\rm rm}\xspace}
\newcommand{\featfg}{{\bf Z}_{\rm fg}\xspace} 

\newcommand{\exactfeatrm}{{\bf Z}_{\rm rm}^{\rm exact}\xspace}
\newcommand{\prefeatrm}{{\bf Z}_{\rm rm}^{\rm pre}\xspace}
\newcommand{\exactfeatfg}{{\bf Z}_{\rm fg}^{\rm exact}\xspace}
\newcommand{\prefeatfg}{{\bf Z}_{\rm fg}^{\rm pre}\xspace}

\newcommand{\projrm}{{\bf U}_{\rm rm}\xspace}

\newcommand{\objfg}{{J}_{\rm fg}\xspace}
\newcommand{\objrm}{{J}_{\rm rm}\xspace}
\newcommand{\objfgp}{{J}_{\rm fgp}\xspace}

\title{Machine Unlearning in Low-Dimensional\\Feature Subspace}

\author{%
{\bfseries Kun Fang$^1$\thanks{Equal contribution. \qquad $^\dagger$Corresponding authors.} \quad Qinghua Tao$^2$\footnotemark[1] \quad Junxu Liu$^{1\dagger}$ \quad Yaxin Xiao$^{1\dagger}$}\\[0.15em]
{\bfseries Yulin Jin$^1$ \quad Qingqing Ye$^1$ \quad Jian Sun$^{2\dagger}$ \quad Haibo Hu$^{1}$}\\[0.55em]
{\small\normalfont
$^1$The Hong Kong Polytechnic University \qquad
$^2$Beijing Institute of Technology}\\[0.35em]
{\footnotesize\normalfont\texttt{\{kun.fang,junxu.liu,samin-yaxin.xiao,qqing.ye,haibo.hu\}@polyu.edu.hk}}\\[-0.20em]
{\footnotesize\normalfont
\texttt{\{qinghua.tao,sunjian\}@bit.edu.cn}
\qquad
\texttt{yulin.jin@connect.polyu.hk}}
}

\begin{document}

\maketitle
\setcounter{footnote}{0}

\begin{abstract}
Machine Unlearning (MU) aims at removing the influence of specific data from a pretrained model while preserving  performance on the remaining data.
In this work, a novel perspective for MU is presented upon low-dimensional feature subspaces, which gives rise to the potentials of separating the remaining and forgetting data herein.
This separability motivates our \ours, a method that proceeds unlearning in a LOw-dimensional FeaTure subspace from the pretrained model through principal projections, which are optimized to maximally capture the information of the remaining data and meanwhile diminish that of the forgetting data. 
During unlearning, \ours simply optimizes a small-size projection matrix that can be flexibly inserted into the pretrained model, requiring just one-shot feature fetching from the pretrained model.
\ours demonstrates that effective and efficient MU can be achieved in low-dimensional feature subspaces even without repeated raw-data access or updates to the entire pretrained model.
Extensive experiments validate the significantly lower computational overhead and superior unlearning performance of \ours across diverse models, datasets, tasks, and applications. 
Code is publicly available.\footnote{\url{https://github.com/fanghenshaometeor/unlearning-feat-subspace}}
\end{abstract}

\section{Introduction}
\label{sec:intro}
Deep Neural Networks (DNNs) pretrained on large-scale data, e.g., web-crawled data,  greatly advocate various applications \cite{schuhmann2022laion,zha2025data}. However, it also raises significant privacy concerns due to its heavy reliance on massive data.
Modern data regulatory frameworks, such as GDPR \cite{hoofnagle2019european}, emphasize to preserve {\it the right to be forgotten} when requesting to delete particular data, e.g., private user data.
To address this challenge, {Machine Unlearning (MU)} \cite{bourtoule2021machine} has emerged and  attracted increasing interests in recent years.
MU targets dual goals: erasing the knowledge from the deleted (forgotten) data and meanwhile preserving the model performance.
A gold-standard approach to MU is to retrain the model from scratch based on the remaining data only, i.e., the entire training data excluding the forgetting data \cite{thudi2022necessity}, which is called  {\it the exact MU}. 
Such a retraining process is computationally expensive, especially for modern DNNs which are commonly large-scale models trained on massive data.
Hence, a surge of research on {\it approximate MU} \cite{izzo2021approximate} has been developed to approach the ideal performance of the retrained model.

\begin{figure*}[t]
    \centering
    \includegraphics[width=0.99\linewidth]{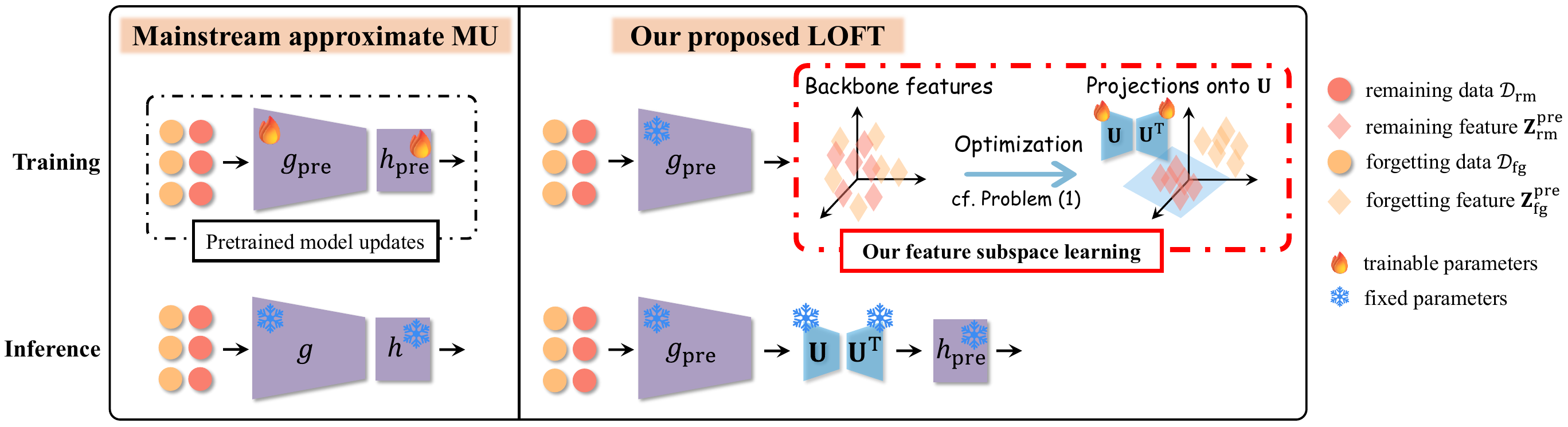}
    \caption{Overview of our proposed \ours and mainstream approximate MU methods.
    After unlearning, the trained projection matrix of \ours is incorporated into the model's forward propagation during inference, with details in Section \ref{sec:sun:implementation}.}
    \label{fig:intro:pipeline}
\end{figure*}

Existing approximate MU methods generally update the parameters of the original model pretrained on the entire training data $\mathcal{D}$, so as to align its outputs with that of the ideal model retrained in the exact MU.
The core idea here is to perform parameter updates with an optimization objective that jointly penalizes the  performance on the forgetting data $\fgdata$ and maintains that on the remaining data $\rmdata$, with $\fgdata \bigcup \rmdata = \mathcal{D}$ and $\fgdata \bigcap \rmdata = \emptyset$.
This optimization objective leads to different ways to modify the pretrained model, including fine-tuning  only on $\rmdata$ \cite{warnecke2021machine}, assigning random labels to $\fgdata$ \cite{golatkar2020eternal}, conducting gradient ascent on $\fgdata$ \cite{thudi2022unrolling}, masking those parameters sensitive to $\fgdata$ \cite{fansalun}, distilling an additional network \cite{chundawat2023can,zhou2025decoupled}, etc. 
Detailed explanations and discussions on these MU methods are provided in Appendix \ref{app:sec:related-work}.
While enabling good performance, we note their two critical drawbacks far underexplored:
\\
{$\bullet$} {\bf 
Frequent revisits to the massive training data.}
Existing MU methods generally request to visit training samples in both $\rmdata$ and $\fgdata$ during the iterative parameter updates to the pretrained model.
Such massive data fetching not only incurs expensive computational overhead, but also poses privacy threats during the frequent access to (user) data.
\\
{$\bullet$} {\bf Iterative modifications to the pretrained model.}
Whenever a particular dataset $\fgdata$ is requested to be forgotten, parameters of the pretrained model have to be updated, yielding a new model tailored to $\fgdata$.
This can be computationally infeasible in practice, as each unlearning request demands a new model modified from the pretrained one.

In this work, we aim to address the above challenges in approximate MU methods.
{\it (i)} Firstly, a new perspective to study MU in {\it low-dimensional feature subspaces} is introduced.
We reveal the potential separability in feature subspaces between $\fgdata$ and $\rmdata$ on the retrained model when comparing to the pretrained model.
{\it (ii)} Then, the observed separability inspires us to {\it learn a feature subspace} from the pretrained model to achieve unlearning.
To this end, the feature subspace is optimized to differentiate $\fgdata$ and $\rmdata$ by capturing maximal information of the features w.r.t. $\rmdata$ and meanwhile describing minimal information w.r.t. $\fgdata$, thus fulfilling the objective of unlearning.
This optimization is implemented following the spirits of Principal Component Analysis (PCA) \cite{pearson1901liii}: 
A projection matrix is trained based on only two feature covariance matrices of $\fgdata$ and $\rmdata$ and then is inserted into the pretrained model for inference.
Our unlearning method is thereby named {\bf\ours} for learning {\bf LO}w-dimensional {\bf F}ea{\bf T}ure subspaces.

Figure \ref{fig:intro:pipeline} provides an overview of our \ours and mainstream approximate MU methods.
\ours offers efficiency for computation and storage, flexibility in practice, and reduced raw-data exposure during unlearning, towards closing the gaps in adapting to real-world MU scenarios:
\\
{$\bullet$} {\bf One-shot feature fetching.} 
\ours conducts MU in the feature space by fetching the features from the pretrained model {\it once}. 
After this one-shot fetching, two covariance matrices of $\fgdata$ and $\rmdata$ are computed and remain fixed, and our PCA-based optimization then updates only the projection directions to learn a subspace that distinguishes $\rmdata$ from $\fgdata$. 
Unlike other approximate MU methods, \ours avoids direct, repeated, iterative raw training data access during unlearning, improving computational efficiency.
\\
{$\bullet$} {\bf Plug-in module implementation.}
\ours optimizes a projection matrix to eliminate information of $\fgdata$ and preserve that of $\rmdata$, thereby enabling a plug-and-play implementation: this projection matrix is inserted into the pretrained model, without retraining or fine-tuning the pretrained parameters.
This merit is significant for efficiency and real-world practicality. 
For instance, when handling multiple unlearning requests from different users, \ours only requires a separate projection module in a rather small size for each request, instead of a new model modified from the pretrained DNN.

The key contributions of this work are summarized below:
\begin{itemize}[leftmargin=20pt]
\item We present a novel feature-subspace perspective for approximate MU and, to the best of our knowledge, are the first to show that forgetting and remaining data can exhibit separability in low-dimensional feature subspaces. 
\item A new MU method named \ours is proposed by learning a feature subspace from pretrained features via PCA-based techniques to capture information from the remaining data while suppressing information from the forgetting data.
\item \ours performs {\it one-shot feature fetching} and {\it optimizes only a projection matrix} using two fixed feature covariance matrices. 
This plug-in design avoids repeated raw-data access and full-model updates, reducing optimized parameters, computation, and per-request storage.
\item Extensive experiments verify our superior unlearning accuracy and efficiency on varied datasets, networks, unlearning tasks, and real-world applications.
\end{itemize}
\section{Preliminaries}
\label{sec:preliminary}
\paragraph{Machine unlearning.}
MU aims at maintaining prediction performance of a well-trained model $\pref$ when removing (or forgetting) specific data  \cite{bourtoule2021machine}. 
In MU, the model $\pref$ pretrained on the full dataset ${\cal D}$ is given.
A particular subset $\fgdata$ is requested to be forgotten, and we have the remaining data denoted as $\rmdata:={\cal D} \backslash\fgdata$.
MU seeks to erase the knowledge of $\fgdata$ from $\pref$ and meanwhile to preserve the performance on $\rmdata$.
The golden standard in MU is to attain a model $\exactf$ retrained from scratch  on $\rmdata$, i.e., the exact MU \cite{thudi2022necessity}.
In contrast, fine-tuning the pretrained model $\pref$ to approximate the ideal performance of $\exactf$ has been a central research focus for flexibility and efficiency, namely the approximate MU \cite{izzo2021approximate}.
We provide more discussions on  related work in Appendix~\ref{app:sec:related-work}.

\paragraph{Principal component analysis.}
PCA has long been a fundamental tool in machine learning and beyond \cite{pearson1901liii}.
PCA seeks orthogonal directions, i.e., principal components capturing the highest projection variance, so as to keep maximal data information. 
These directions can be obtained by taking the eigenvectors of the input covariance matrix. 
PCA can explore the intrinsic patterns residing in low-dimension subspaces spanned by those principal components, and has been studied for learning with modern DNNs, e.g., training algorithms \cite{dldr}  and anomaly detection applications \cite{fang2024kpcaood,kpcaood2026}. 
In this work, we follow the spirits of PCA for MU by {\it learning} the projection directions associated with the low-dimensional feature subspace from $\pref$. 

\paragraph{Notations.}
Let $\pref: \mathcal{X}\to\mathbb R^C$ be a DNN pretrained on  ${\cal D}=\{\boldsymbol{x}_i,y_i\}_{i=1}^{N}\subseteq\mathcal{X}\times\mathcal{Y}$, with samples $\boldsymbol{x}_i \in\mathcal{X}\subset\mathbb{R}^D$ and their labels $y_i\in\mathcal{Y}=\{1,2,\cdots,C\}$. 
The model $f(\cdot)$ takes $\boldsymbol{x}$ as the input and outputs $C$-dimensional logits $f(\boldsymbol{x})\in\mathbb{R}^C$ for classification prediction.
More specifically, $f$ can be structured with a backbone $g:\mathcal{X}\rightarrow\mathbb{R}^d$ and a linear layer $h:\mathbb{R}^d\rightarrow\mathbb{R}^C$ in sequence, where
$g(\cdot)$ learns the (penultimate-layer)  feature $\boldsymbol{z}=g(\boldsymbol{x})\in\mathbb{R}^d$ and then the linear layer $h(\cdot)$ is applied to the feature $\boldsymbol{z}$ for the output logits, i.e., $f(\boldsymbol{x})=h(g(\boldsymbol{x}))$.
\section{A Feature Subspace Perspective for Approximate MU}
\label{sec:low-feat-subspace}
In this section, we present an empirical study that motivates our feature subspace perspective for approximate MU. 
Particularly, we study whether $\fgdata$ and $\rmdata$ can be separated in low-dimensional feature subspaces under $\pref$ and $\exactf$.
Here, ``separability'' does not refer to linear classification separability, but to a reconstruction property in low-dimensional feature subspaces, e.g., a subspace fitted to $\rmdata$ reconstructs $\rmdata$ 
significantly better than $\fgdata$.

The pretrained $\pref$ is learned with the entire training dataset $ \mathcal{D} = \fgdata \bigcup \rmdata$, whereas the ideal model $\exactf$ from the exact MU is trained exclusively on $\rmdata$ and never sees $\fgdata$. 
This fundamental distinction in their training paradigms naturally raises the following question:\\
{\it After exact retraining, do the features of $\fgdata$ and $\rmdata$ exhibit different low-dimensional structures, and if so, can this low-dimensional property inspire approximate MU from the pretrained model?}

\begin{figure*}[t]
    \centering
    \includegraphics[width=0.85\linewidth]{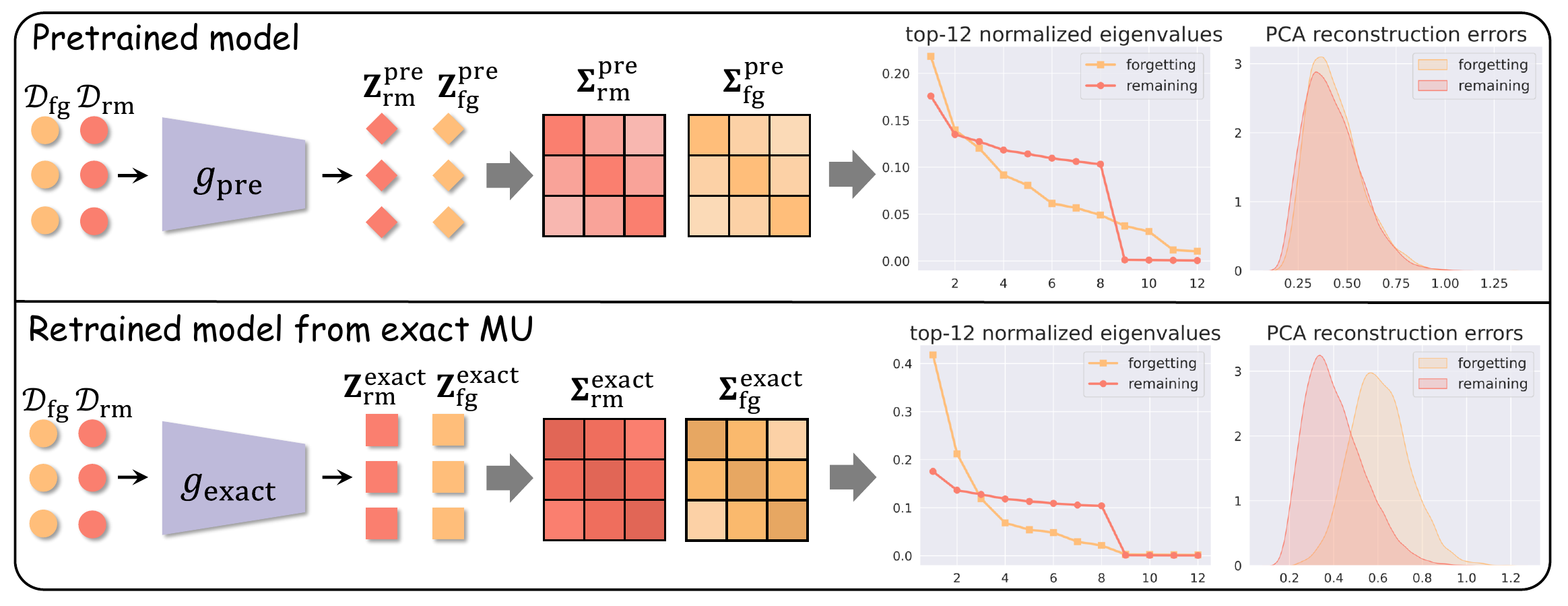}
    \caption{Spectrum and reconstruction analyses on features of $\fgdata$ and $\rmdata$ w.r.t. the pretrained $\pref$ and the retrained $\exactf$, with results on the top-12 normalized eigenvalues and reconstruction errors. 
    Detailed setups and more empirical results for this analysis can be found in Appendix \ref{app:sec:spectrum-analysis}.
    }
    \label{fig:spectrum}
\end{figure*}

\paragraph{Setups.}
In our analysis, we consider the penultimate-layer features $\boldsymbol{z}\in\mathbb R^d$ from the network backbone.
Let $\prefeatfg,\prefeatrm,\exactfeatfg,\exactfeatrm$ denote the features of $\fgdata$ and $\rmdata$ from the pretrained $\preg$ and the retrained $\exactg$, respectively.
Their corresponding covariance matrices are denoted by $\presigmafg,\presigmarm,\exactsigmafg,\exactsigmarm\in\mathbb{R}^{d\times d}$.
With eigen-decomposition to those covariance matrices, we investigate their eigenvalues and reconstruction performance, demonstrated in Figure \ref{fig:spectrum}.

\paragraph{Spectrum analysis.}
Figure \ref{fig:spectrum} visualizes the eigen-decay, i.e., the distribution of variance across principal components, for the aforementioned covariance matrices. 
In the pretrained $\pref$, the eigenvalue decays of $\presigmafg$ and $\presigmarm$ exhibit similar trends. 
In contrast, a marked divergence appears in the retrained $\exactf$: $\exactsigmafg$ shows a distinctly sharper eigenvalue decay than $\exactsigmarm$. 
Intuitively, since $\exactf$ is trained exclusively on $\rmdata$, its feature space is optimized solely to capture the variations of $\rmdata$, lacking dedicated directions to encode the unique characteristics of the unseen $\fgdata$. 
Consequently, features of unseen $\fgdata$ tend to collapse into fewer dominant directions, leading to the sharper decay, while $\rmdata$ retains a more dispersed structure. 

\paragraph{Reconstruction analysis.}
The PCA reconstruction error quantitatively measures the captured information along projection directions.
We apply PCA to the remaining features $\featrm$ and obtain $\projrm\in \mathbb R^{d\times s}$ projecting features onto the $s$-dimensional subspace.
The reconstruction error for a centered feature $\boldsymbol{\bar z}$ is calculated as $e=\|\projrm\projrm^\top{\boldsymbol{\bar z}}-\boldsymbol{\bar z}\|_2$. 
In Figure \ref{fig:spectrum}, we compute such errors on both $\featrm$ and $\featfg$ under $\pref$ and $\exactf$, respectively.
The reconstruction errors for $\prefeatfg$ and $\prefeatrm$ of the pretrained $\pref$ are very similar, while those for $\exactfeatfg$ and $\exactfeatrm$ of the retrained $\exactf$ appear distinctively different.
We provide more empirical results for this analysis in Appendix \ref{app:sec:spectrum-analysis}.

\paragraph{Discussion.}
This analysis suggests an answer to the question raised earlier: under exact retraining, a feature subspace aligned with $\rmdata$ captures its dominant structure well but describes $\fgdata$ far less effectively, whereas such separability is considerably weaker in the pretrained $\pref$.
Inspired by this observation, we seek to {\it learn} a feature subspace from $\pref$.
Note that our goal is not to recover the ideal subspace from $\exactf$, but rather to execute subspace learning based on $\pref$ toward a similar objective: a feature subspace that preserves dominant variation of $\rmdata$, while suppressing information associated with $\fgdata$, thereby achieving effective unlearning on $\pref$.

\paragraph{Remark.}
We emphasize that our analysis here is intended primarily for the class-wise forgetting.
In this setting, $\fgdata$ includes samples from one or multiple classes, and is thereby semantically more coherent, making its feature structure more likely to deviate from that of $\rmdata$ after exact retraining.
In contrast, under the more challenging instance-wise forgetting, $\fgdata$ and $\rmdata$ may substantially overlap in semantics and class membership, leading to weaker subspace separability, as demonstrated in Appendix \ref{app:sec:spectrum-analysis}.
Accordingly, the purpose of this analysis is not to claim universal separability across all unlearning scenarios, but to identify a practically useful feature subspace phenomenon that motivates our method.
Moreover, our principle of subspace learning is also potentially beneficial for instance-wise unlearning, where the desired separability is weaker and must be induced rather than expected to emerge directly, as validated in our experiments later in \cref{sec:exp}.
\section{Approximate MU with Subspace Learning}
\label{sec:method}

Grounded on the feature subspace perspective and the analysis in Section \ref{sec:low-feat-subspace}, we now focus on how to attain a preferable feature subspace for $\pref$, such that the features $\prefeatrm$ and $\prefeatfg$ can be well distinguished in this subspace, i.e., diminishing the knowledge of $\fgdata$ and preserving that of $\rmdata$ from $\pref$. 
Accordingly, we name this proposed method as {\bf\ours}, learning {\bf LO}w-dimensional {\bf F}ea{\bf T}ure subspaces. 
Apart from approaching the outputs or parameters of $\exactf$, \ours leads to an interesting and promising direction of formulating unlearning models that pertain similar properties in feature (sub)spaces, which we hope could bring new perspectives and opportunities in MU.

\subsection{Optimization}
\ours aims to find an $s$-dimensional subspace ${\tt range}(\bf U)$ for the $d$-dimensional features of the pretrained model $\pref$.
This linear subspace is given by the span of the orthonormal columns of the $d\times s$ matrix ${\bf U}=[\boldsymbol{u}_1, \ldots, \boldsymbol{u}_s]$, the set of which defines the well-known Stiefel manifold ${\rm St}(d,s)$ with $d \geq s$. Towards the feature separability for MU, \ours follows the spirits of PCA to seek such a subspace (to optimize $\bf U$),  in which the information of $\prefeatrm$ is maximally captured and meanwhile that of $\prefeatfg$ gets minimally described.
We start from Proposition \ref{prop:min:evd} \cite{arun} below.

\begin{proposition}[\cite{arun}]
\label{prop:min:evd}
Let $\bf M$ be an $l \times l$ symmetric matrix. 
Let $\gamma_1, \ldots, \gamma_m$  be its $m$ smallest eigenvalues, possibly including multiplicities, with associated orthonormal eigenvectors $\boldsymbol{v}_1, \ldots, \boldsymbol{v}_m$. 
Let $\bf V$ be a matrix whose columns are these eigenvectors. 
Then, the optimization problem $\min_{{\bf U}\in {\rm St}(l, m)} {\rm Tr}({\bf U}^\top {\bf M} {\bf U})$ has a minimizer at ${\bf U}_\star = {\bf V}$ and we have ${\bf U}_\star^\top {\bf M} {\bf U}_\star = {\rm diag}({\bf\Gamma})$ with ${\bf\Gamma} = [\gamma_1, \ldots, \gamma_m]^\top$, where $\rm Tr(\cdot)$ denotes the matrix trace.
\end{proposition}
Proposition \ref{prop:min:evd} \cite{arun} formalizes the optimization problem for obtaining the projection directions capturing minimal  information in PCA with the Stiefel manifold, where $\bf M$ denotes the covariance matrix. 
This correspondingly leads to the optimization problem for  the projection directions capturing maximal information: the subspace spanned by the eigenvectors of $\bf M$ with the $m$ largest eigenvalues is obtained by solving 
$  \min_{{\bf U}\in {\rm St}(l, m)} {\rm Tr}({\bf M} - {\bf U}{\bf U}^\top {\bf M} {\bf U} {\bf U}^\top)$.

While Proposition \ref{prop:min:evd} \cite{arun} identifies the optimal subspace for a single criterion (either minimizing or maximizing variance), it motivates the optimization objective of \ours to address the dual requirements in MU of suppressing $\fgdata$ ($\objfg$) and preserving $\rmdata$ ($\objrm$) as follows.
\begin{equation}
\begin{aligned}
\label{eq:opt-obj}
\min_{{\bf U}\in{\rm St}(d,s)}
{J}({\bf U})
=
\underbrace{
\Bigl(
\frac
{{\rm Tr}({\bf U}^\top\presigmafg{\bf U})}
{{\rm Tr}(\presigmafg)}
\Bigr)^2}_{\objfg}
+
\underbrace{
\Bigl(
\frac
{{\rm Tr}\left(\presigmarm-{\bf U}{\bf U}^\top\presigmarm{\bf U}{\bf U}^\top\right)}
{{\rm Tr}\left(\presigmarm\right)}
\Bigr)^2}_{\objrm},
\end{aligned}
\end{equation}
where numerators play key roles and denominators balance normalization.
The joint optimization of $\objrm$ and $\objfg$ advocates the separability between $\prefeatrm$ and $\prefeatfg$ in the subspace projected via $\bf U$:
\begin{itemize}[leftmargin=20pt]
\item $\objfg$ quantifies the projection variance of the forgetting features $\prefeatfg$ onto the subspace spanned by $\bf U$, as explained in Proposition \ref{prop:min:evd} \cite{arun}. 
Minimizing $\objfg$ encourages $\bf U$ to capture minimal information of  $\prefeatfg$, aiming to forget the knowledge of $\fgdata$. 
\item $\objrm$ measures the projection variance of the remaining features $\prefeatrm$ captured within the complement subspace w.r.t. $\bf U$, namely the reconstruction error.
Minimizing $\objrm$ forces $\bf U$ to preserve maximal information of $\prefeatrm$, so as to maintain model performance on $\rmdata$.
\end{itemize}

The Riemannian Adam optimizer \cite{becigneulriemannian,kochurov2020geoopt} is deployed to solve \eqref{eq:opt-obj} with $\bf U$ optimized on the Stiefel manifold ${\rm St}(d,s)$.
We show that $J({\bf U})$ is $L$-smooth, and thereby the convergence rate of Riemannian Adam on minimizing $J({\bf U})$ is guaranteed in Lemma \ref{lm:convergence} of Appendix \ref{app:sec:convergence}.

\begin{algorithm}[t]
\caption{\ours for Approximate Machine Unlearning}
\label{alg:loft}
\begin{algorithmic}[1]

\Statex \textbf{Offline one-shot feature fetching}
\Require Pretrained model $\pref(\cdot)=\preh(\preg(\cdot))$ and training set $\mathcal D=\{(\boldsymbol x_i,y_i)\}_{i=1}^{N}$
\Ensure Indexed training feature set $\mathcal Z$

\State Set $\pref$ to evaluation mode and freeze all its parameters
\State Extract $\boldsymbol z_i\gets\preg(\boldsymbol x_i)$ for every $(\boldsymbol x_i,y_i)\in\mathcal D$ through one inference pass
\State Store the indexed training feature set $\mathcal Z\gets\{(i,\boldsymbol z_i)\}_{i=1}^{N}$

\Statex
\Statex \textbf{On-demand unlearning}
\Require Pretrained model $\pref$, training feature set $\mathcal Z$, forgetting request $\mathcal F\subseteq\{1,\ldots,N\}$, subspace dimension $s$, and optimization steps $T$
\Ensure Request-specific projection $\mathbf U$ and unlearned model $f_{\mathbf U}$

\State Retrieve forgetting features $\prefeatfg\gets\{\boldsymbol z_i\mid i\in\mathcal F\}$ from $\mathcal Z$
\State Retrieve remaining features $\prefeatrm\gets\{\boldsymbol z_i\mid i\in\{1,\ldots,N\}\setminus\mathcal F\}$ from $\mathcal Z$
\State Compute covariances $\presigmafg\gets\operatorname{Cov}(\prefeatfg)$ and $\presigmarm\gets\operatorname{Cov}(\prefeatrm)$
\State Initialize $\mathbf U\in{\rm St}(d,s)$
\For{$t=1,\ldots,T$}
\State Evaluate $J(\mathbf U)=\objfg+\objrm$ using $\presigmafg$ and $\presigmarm$ as defined in \eqref{eq:opt-obj}
\State Update $\mathbf U$ using Riemannian Adam on ${\rm St}(d,s)$ \Comment{No iterative raw-data access}
\EndFor

\State \Return $\mathbf U$ and $f_{\mathbf U}(\cdot)\triangleq\preh(\mathbf U\mathbf U^\top\preg(\cdot))$

\end{algorithmic}
\end{algorithm}

\subsection{Implementation and Discussion}
\label{sec:sun:implementation}

\ours simply requires to optimize a projection matrix $\bf U$ applied to features $\prefeatfg$ and $\prefeatrm$ from the pretrained $\pref$. 
During inference, the resulting model from \ours for unlearning is given by $f_{\rm\bf U}(\cdot)\triangleq \preh({\bf U}{\bf U}^\top \preg(\cdot))$, in contrast to the pretrained $\pref =  \preh(\preg(\cdot))$, as shown in Figure \ref{fig:intro:pipeline}.
Algorithm \ref{alg:loft} provides an intuitive overview of a practical implementation of \ours.

\paragraph{One-shot feature fetching.}
The computation to ${J}({\bf U})$ relies on the two covariance matrices $\presigmarm$ and $\presigmafg$ w.r.t. features from $\pref$. 
Once $\presigmarm,\presigmafg$ are computed, \ours proceeds to optimize the $d\times s$ projection matrix $\bf U$ and thus only conducts one-shot feature fetching. 
This differs fundamentally from existing approximate MU methods, which must repeatedly access the original training data at each optimization iteration.
By eliminating the need for such direct and massive visits to data samples, \ours significantly reduces unlearning-time computation and raw-data exposure.

\paragraph{Plug-in module implementation.} 
Compared to fine-tuning or retraining all parameters in $\pref$, \ours optimizes a $d\times s$ matrix as in (\ref{eq:opt-obj}), the parameter number of which is thereby reduced by orders of magnitude. 
Thus, \ours takes significantly less computation overhead than the exact MU and mainstream approximate MU methods, supported by numerical results in Section \ref{sec:exp:comparison}.
In implementation, \ours flexibly serves as a plug-in module into $\pref$, without modifying parameters of $\pref$.
This merit is of great practicality in real-world deployment: in cases of handling multiple unlearning requests from different users, each request can be addressed by efficiently optimizing a small-size projection matrix, instead of maintaining multiple models refined from $\pref$.

\paragraph{Theoretical analysis.}
\ours offers a novel feature subspace perspective and an empirical method for MU.
We position \ours as orthogonal and complementary to the established research on certified unlearning \cite{pmlr-v119-guo20c}. 
Certified methods provide rigorous guarantees, often via differential privacy \cite{dwork2006differential} by bounding the parameter difference between approximate and exact unlearned models.
In contrast, \ours operates on a fundamentally different principle.
Instead of updating parameters of $\pref$, \ours applies a learnable projection matrix $\bf U$ on features from $\pref$ to achieve unlearning in a feature subspace sense.
As a result, existing certification frameworks based on parameter-space proximity are not directly applicable to \ours.

Nevertheless, we provide a preliminary theoretical discussion for \ours, as formalized in Proposition \ref{prop:loft-gap} below.
It shows that, under bounded optimization errors and proximity assumptions between $\pref$ and $\exactf$, the output discrepancy between our $f_{\bf U}$ and $\exactf$ can be controlled.
Proposition~\ref{prop:loft-gap} should therefore be interpreted only as a conditional discrepancy analysis that offers intuition for the approximation behavior of \ours, rather than a certified unlearning guarantee.
Due to space limitation, a detailed version of Proposition~\ref{prop:loft-gap} is put in Appendix \ref{app:sec:unlearning}.
\begin{proposition}
\label{prop:loft-gap}
Suppose the learned projection $\bf\hat U$ satisfies ${\rm Tr}({\bf\hat U}^\top\presigmafg{\bf\hat U})\leq\epsilon_{\rm fg}$ and ${\rm Tr}(({\bf I}-{\bf\hat U}{\bf\hat U}^\top)\presigmarm({\bf I}-{\bf\hat U}{\bf\hat U}^\top))\leq\epsilon_{\rm rm}$.
Then, under explicit boundedness and proximity assumptions stated in Appendix \ref{app:sec:unlearning}, the expected output discrepancy between $f_{\bf U}$ and $\exactf$, i.e., $\mathbb{E}_{\boldsymbol{x}}
\|f_{\bf\hat U}(\boldsymbol{x})-\exactf(\boldsymbol{x})\|_2$, is bounded by terms that scale with $\epsilon_{\rm rm}$ for $\boldsymbol{x}\in\rmdata$ and with $\epsilon_{\rm fg}$ for $\boldsymbol{x}\in\fgdata$.
\end{proposition}

\section{Experiments}
\label{sec:exp}

In this section, extensive experiments are conducted to evaluate \ours, showing its potentials of feature subspace learning for MU.
In experiments, we primarily focus on the {\it class-centric} unlearning \cite{zhou2025decoupled}, widely studied in real-world scenarios, as data from different users can represent varied classes.
Another significant {\it instance-wise} unlearning is also investigated.
More empirical results, details and discussions are provided in Appendices \ref{app:sec:loft-efficiency}-\ref{app:sec:generation}, including detailed efficiency analysis, more unlearning scenarios, real-world applications and generation tasks.

\subsection{Setups}

\paragraph{Datasets and models.}
Experiments are executed on Tiny-ImageNet \cite{le2015tiny} (200 classes, $64\times64\times3$) and ImageNet-1K \cite{deng2009imagenet} (1,000 classes, $224\times224\times3$).
We skip comparisons on other easier datasets with smaller scales and distinctively fewer classes, e.g., CIFAR \cite{krizhevsky2009learning} and SVHN \cite{netzer2011reading}, which have been extensively evaluated and yet show almost incremental improvements with limited challenges for class-centric unlearning.
Swin-T \cite{liu2021swin} and ResNet50 \cite{he2016deep} are evaluated on Tiny-ImageNet and ImageNet-1K, respectively.
Main comparison experiments are repeated 3 times, with mean and standard deviations reported.

\paragraph{Baselines and restricted data access.}
We compare \ours with a series of strong baselines: Gradient Ascent (GA) \cite{thudi2022unrolling}, Random Labeling (RL) \cite{golatkar2020eternal}, SalUn \cite{fansalun}, Learn to UnLearn (L2UL) \cite{cha2024learning}, COUN \cite{khalilcoun} and DELETE \cite{zhou2025decoupled}. 
Results of the pretrained $\pref$ and the retrained $\exactf$ are also included.
See Appendices \ref{app:sec:related-work} and \ref{app:sec:settings} for baseline outlines and implementation details.

We consider a practically relevant MU protocol where repeated access to raw training samples is {\it restricted} during unlearning, as iteratively revisiting raw $\fgdata$ and $\rmdata$ can be costly and undesirable from a privacy perspective. 
Existing methods are typically evaluated either with full access to both datasets or in a more restrictive setting using only raw $\fgdata$ \cite{cha2024learning,zhou2025decoupled}. 
In contrast, \ours does not revisit raw $\fgdata$ or $\rmdata$ during unlearning, but uses their feature covariances, $\presigmafg$ and $\presigmarm$, precomputed from a one-shot forward pass of $\pref$. 
Thus, \ours still leverages information from $\fgdata$ and $\rmdata$, but under a summarized-data access protocol rather than raw-data iterative optimization.
Accordingly, we report restricted-access baselines using only raw $\fgdata$. 
For reference, we include baselines with full access to both $\fgdata$ and $\rmdata$ (marked in {\textcolor{gray}{gray}}), which are not directly matched to \ours in access modality but help contextualize the gap between restricted raw-data access and full-data optimization.

\paragraph{Metrics.}
Four metrics on learning accuracy are leveraged \cite{zhou2025decoupled}:
$\rm Acc_{fg}^{tr}$ on the forgetting $\fgdata$ and $\rm Acc_{rm}^{tr}$ on the remaining  $\rmdata$ of the training data, $\rm Acc_{fg}^{te}$ on the  forgetting  $\fgdata^{\rm te}$ and $\rm Acc_{rm}^{te}$ on the remaining  $\rmdata^{\rm te}$ of the test data.
In addition, the Membership Inference Attack (MIA) success rate \cite{shokri2017membership} is adopted, measuring the proportion of samples in $\fgdata$ that gets memorized by the unlearned model. 
Results against more MIAs, LiRA\cite{carlini2022membership} and RMAI \cite{zarifzadeh2024low}, are provided in Appendix \ref{app:sec:privacy-threat}.

Note that the performance gap to the retrained $\exactf$ essentially reflects the unlearning efficacy of approximate MU methods. 
This is the key evaluation criterion for MU and is marked in {\textcolor{blue}{blue}} within parentheses. 
Hence, the overall performances are indicated by the average of the accuracy and MIA gap to $\exactf$ (Avg.G.).
We mark the best results in bold font. 

\begin{table*}[t]
\centering
\caption{Comparison results of {Swin-T} on {Tiny-ImageNet}. 
The $\checkmark$ and \ding{55} indicate whether $\rmdata$ and $\fgdata$ are involved in optimization.
Avg.G. is average accuracy and MIA gap to the retrained model.}
\resizebox{\textwidth}{!}{
\begin{tabular}{c|cc|cc cc|c| c}
\toprule
method & $\rmdata$ & $\fgdata$ & $\rm\bf Acc_{rm}^{tr}$ & $\rm\bf Acc_{fg}^{tr}$ & $\rm\bf Acc_{rm}^{te}$ & $\rm\bf Acc_{fg}^{te}$ & {\bf MIA} & Avg.G.$\downarrow$\\
\midrule
$\pref$ & $\usym{1F5F8}$ & $\usym{1F5F8}$ & $99.63_{\pm0.01}$ & $99.75_{\pm0.20}$ & $74.54_{\pm0.16}$ & $74.83_{\pm8.14}$ & $96.63_{\pm1.02}$ & - \\
$\exactf$ & $\usym{1F5F8}$ & $\usym{1F5F4}$ & $99.67_{\pm0.03}$ (\textcolor{blue}{0.00}) & $0.00_{\pm0.00}$ (\textcolor{blue}{0.00}) & $75.44_{\pm0.19}$ (\textcolor{blue}{0.00}) & $0.00_{\pm0.00}$ (\textcolor{blue}{0.00}) & $0.00_{\pm0.00}$ (\textcolor{blue}{0.00}) & \textcolor{blue}{0.00}\\
\midrule
GA & $\usym{1F5F4}$ & $\usym{1F5F8}$ & $88.40_{\pm5.51}$ (\textcolor{blue}{11.27}) & $23.12_{\pm12.14}$ (\textcolor{blue}{23.12}) & $65.27_{\pm3.56}$ (\textcolor{blue}{10.17}) & $18.82_{\pm7.52}$ (\textcolor{blue}{18.82}) & $19.95_{\pm10.24}$ (\textcolor{blue}{19.95}) & \textcolor{blue}{16.67}\\
\textcolor{gray}{GA} & \textcolor{gray}{$\usym{1F5F8}$} & \textcolor{gray}{$\usym{1F5F8}$} & \textcolor{gray}{$99.49_{\pm0.03}$ (\textcolor{grayblue}{0.18})} & \textcolor{gray}{$0.02_{\pm0.03}$ (\textcolor{grayblue}{0.02})} & \textcolor{gray}{$74.57_{\pm0.21}$ (\textcolor{grayblue}{0.87})} & \textcolor{gray}{$0.00_{\pm0.00}$ (\textcolor{grayblue}{0.00})} & \textcolor{gray}{$0.00_{\pm0.00}$ (\textcolor{grayblue}{0.00})} & \textcolor{grayblue}{0.21}\\
RL & $\usym{1F5F4}$ & $\usym{1F5F8}$ & $92.94_{\pm1.91}$ (\textcolor{blue}{6.73}) & $9.98_{\pm7.44}$ (\textcolor{blue}{9.98}) & $67.17_{\pm1.46}$ (\textcolor{blue}{8.17}) & $7.50_{\pm7.57}$ (\textcolor{blue}{7.50}) & $1.83_{\pm0.54}$ (\textcolor{blue}{1.83}) & \textcolor{blue}{6.85} \\
\textcolor{gray}{RL} & \textcolor{gray}{$\usym{1F5F8}$} & \textcolor{gray}{$\usym{1F5F8}$} & \textcolor{gray}{$99.91_{\pm0.01}$ (\textcolor{grayblue}{0.24})} & \textcolor{gray}{$1.52_{\pm1.17}$ (\textcolor{grayblue}{1.52})} & \textcolor{gray}{$74.43_{\pm0.18}$ (\textcolor{grayblue}{1.01})} & \textcolor{gray}{$0.33_{\pm0.58}$ (\textcolor{grayblue}{0.33})} & \textcolor{gray}{$0.00_{\pm0.00}$ (\textcolor{grayblue}{0.00})} & \textcolor{grayblue}{0.62}\\
SalUn & $\usym{1F5F4}$ & $\usym{1F5F8}$ & $93.16_{\pm1.90}$ (\textcolor{blue}{6.51}) & $13.55_{\pm10.30}$ (\textcolor{blue}{13.55}) & $67.49_{\pm1.56}$ (\textcolor{blue}{7.95}) & $10.00_{\pm9.26}$ (\textcolor{blue}{10.00}) & $3.45_{\pm0.87}$ (\textcolor{blue}{3.45}) & \textcolor{blue}{8.29} \\
\textcolor{gray}{SalUn} & \textcolor{gray}{$\usym{1F5F8}$} & \textcolor{gray}{$\usym{1F5F8}$} & \textcolor{gray}{$99.88_{\pm0.01}$ (\textcolor{grayblue}{0.21})} & \textcolor{gray}{$2.57_{\pm1.09}$ (\textcolor{grayblue}{2.57})} & \textcolor{gray}{$74.71_{\pm0.19}$ (\textcolor{grayblue}{0.73})} & \textcolor{gray}{$1.00_{\pm0.87}$ (\textcolor{grayblue}{1.00})} & \textcolor{gray}{$0.00_{\pm0.00}$ (\textcolor{grayblue}{0.00})} & \textcolor{grayblue}{0.90} \\
COUN+RL & $\usym{1F5F4}$ & $\usym{1F5F8}$ & $92.56_{\pm0.75}$ (\textcolor{blue}{7.11}) & $10.07_{\pm3.85}$ (\textcolor{blue}{10.07}) & $67.85_{\pm0.51}$ (\textcolor{blue}{7.59}) & $5.67_{\pm2.57}$ (\textcolor{blue}{5.67}) & $2.28_{\pm0.46}$ (\textcolor{blue}{2.28}) & \textcolor{blue}{6.54} \\ 
\textcolor{gray}{COUN+RL} & \textcolor{gray}{$\usym{1F5F8}$} & \textcolor{gray}{$\usym{1F5F8}$} & \textcolor{gray}{$99.94_{\pm0.01}$ (\textcolor{grayblue}{0.27})} & \textcolor{gray}{$0.73_{\pm0.16}$ (\textcolor{grayblue}{0.73})} & \textcolor{gray}{$72.89_{\pm0.18}$ (\textcolor{grayblue}{1.36})} & \textcolor{gray}{$0.67_{\pm0.58}$ (\textcolor{grayblue}{0.67})} & \textcolor{gray}{$0.00_{\pm0.00}$ (\textcolor{grayblue}{0.00})} & \textcolor{grayblue}{0.61}\\
L2UL & $\usym{1F5F4}$ & $\usym{1F5F8}$ & $92.69_{\pm3.43}$ (\textcolor{blue}{6.98}) & $4.43_{\pm0.15}$ (\textcolor{blue}{4.43}) & $67.69_{\pm2.58}$ (\textcolor{blue}{7.75}) & $2.00_{\pm0.87}$ (\textcolor{blue}{2.00}) & $2.57_{\pm0.92}$ (\textcolor{blue}{2.57}) & \textcolor{blue}{4.75}\\
DELETE & $\usym{1F5F4}$ & $\usym{1F5F8}$ & $99.13_{\pm0.12}$ (\textcolor{blue}{\bf{0.54}}) & $13.32_{\pm4.17}$ (\textcolor{blue}{13.32}) & $73.31_{\pm0.28}$ (\textcolor{blue}{2.13}) & $6.50_{\pm2.18}$ (\textcolor{blue}{6.50}) & $1.85_{\pm0.48}$ (\textcolor{blue}{1.85}) & \textcolor{blue}{4.88} \\
\midrule
{\bf\ours} & only $\presigmarm$ & only $\presigmafg$ & $98.59_{\pm0.07}$ (\textcolor{blue}{1.08}) & $1.48_{\pm1.27}$ (\textcolor{blue}{\bf1.48}) & $74.27_{\pm0.18}$ (\textcolor{blue}{\bf1.17}) & $0.50_{\pm0.50}$ (\textcolor{blue}{\bf0.50}) & $0.00_{\pm0.00}$ (\textcolor{blue}{\bf0.00}) & \textcolor{blue}{\bf0.85}\\
\bottomrule
\end{tabular}}
\label{tab:exp-timgnet-swint-class}
\end{table*}

\begin{table*}[t]
\centering
\caption{Comparison on computational costs w.r.t. the results in Table~\ref{tab:exp-timgnet-swint-class}. 
The detailed counts and percentages of optimized parameters are listed.
RTE is the run time efficiency in seconds on single NVIDIA GeForce RTX 4090 GPU.
For \ours, RTE measures only the optimization time of \eqref{eq:opt-obj}, as training features can be pre-calculated in practice and covariances are readily available for on-demand use.
Refer to Table \ref{tab:exp-timgnet-swint-class-complexity-fullpipeline} in Appendix \ref{app:sec:loft-efficiency} for a detailed analysis on the full pipeline of \ours.}
\resizebox{0.9\textwidth}{!}{
\begin{tabular}{c|cc|ccc}
\toprule
method & $\rmdata$ & $\fgdata$ & \# optimized param. (\%) & RTE (seconds) $\downarrow$ & additional requirements \\
\midrule
$\pref$ & $\usym{1F5F8}$ & $\usym{1F5F8}$ & 27,673,154 (100\%) & 882.05 & -\\
$\exactf$ & $\usym{1F5F8}$ & $\usym{1F5F4}$ & 27,673,154 (100\%) & 858.14 & -\\
\midrule
GA & $\usym{1F5F4}$ & $\usym{1F5F8}$ & 27,673,154 (100\%) & 13.49 & -\\
GA & $\usym{1F5F8}$ & $\usym{1F5F8}$ & 27,673,154 (100\%) & 218.37 & -\\
RL & $\usym{1F5F4}$ & $\usym{1F5F8}$ & 27,673,154 (100\%) & 13.19 & -\\
RL & $\usym{1F5F8}$ & $\usym{1F5F8}$ & 27,673,154 (100\%) & 438.48 & -\\
SalUn & $\usym{1F5F4}$ & $\usym{1F5F8}$ & 13,836,577 (50\%) & 15.00 & a parameter saliency mask\\
SalUn & $\usym{1F5F8}$ & $\usym{1F5F8}$ & 13,836,577 (50\%) & 450.23 & a parameter saliency mask\\
COUN+RL & $\usym{1F5F4}$ & $\usym{1F5F8}$ & 27,673,154 (100\%) & 23.45 & - \\
COUN+RL & $\usym{1F5F8}$ & $\usym{1F5F8}$ & 27,673,154 (100\%) & 704.87 & - \\
L2UL & $\usym{1F5F4}$ & $\usym{1F5F8}$ & 27,673,154 (100\%) & 15225.05 & \makecell{a set of adversarial examples;\\a parameter importance mask}\\
DELETE & $\usym{1F5F4}$ & $\usym{1F5F8}$ & 27,673,154 (100\%) & 22.14 & a copy of the pretrained model\\
\midrule
{\bf\ours} & only $\presigmarm$ & only $\presigmafg$ & {\bf192,000} ({\bf0.69\%}) & {\bf0.24} & a projection matrix ${\bf U}\in{\rm St}(768,250)$ \\
\bottomrule

\end{tabular}}
\label{tab:exp-timgnet-swint-class-complexity}
\end{table*}

\begin{table*}[t]
\centering
\caption{Comparison results under single-class unlearning with ResNet50 on ImageNet-1K.}
\resizebox{0.9\textwidth}{!}{
\begin{tabular}{c|cc|cc cc|ccc}
\toprule
method & $\rmdata$ & $\fgdata$ & $\rm\bf Acc_{rm}^{tr}$ & $\rm\bf Acc_{fg}^{tr}$ & $\rm\bf Acc_{rm}^{te}$ & $\rm\bf Acc_{fg}^{te}$ & Avg.G.$\downarrow$ & RTE $\downarrow$& \# param. (\%) \\
\midrule
$\pref$ & $\usym{1F5F8}$ & $\usym{1F5F8}$ & $77.92_{\pm0.01}$ & $83.46_{\pm7.64}$ & $76.15_{\pm0.01}$ & $81.33_{\pm10.26}$ & - & - & 25,557,032 (100\%)\\
$\exactf$ & $\usym{1F5F8}$ & $\usym{1F5F4}$ & $80.25_{\pm0.16}$ (\textcolor{blue}{0.00}) & $0.00_{\pm0.00}$ (\textcolor{blue}{0.00}) & $75.72_{\pm0.04}$ (\textcolor{blue}{0.00}) & $0.00_{\pm0.00}$ (\textcolor{blue}{0.00}) & \textcolor{blue}{0.00} & - & 25,557,032 (100\%)\\
\midrule
GA & $\usym{1F5F4}$ & $\usym{1F5F8}$ & $71.87_{\pm1.80}$ (\textcolor{blue}{8.38}) & $0.00_{\pm0.00}$ (\textcolor{blue}{0.00}) & $71.42_{\pm1.89}$ (\textcolor{blue}{4.30}) & $0.00_{\pm0.00}$ (\textcolor{blue}{0.00}) & \textcolor{blue}{3.17} & 9.32 & 25,557,032 (100\%)\\
RL & $\usym{1F5F4}$ & $\usym{1F5F8}$ & $72.11_{\pm1.42}$ (\textcolor{blue}{8.14}) & $1.87_{\pm0.69}$ (\textcolor{blue}{1.87}) & $71.53_{\pm1.44}$ (\textcolor{blue}{4.19}) & $5.33_{\pm4.16}$ (\textcolor{blue}{5.33}) & \textcolor{blue}{4.88} & 14.40 & 25,557,032 (100\%)\\
SalUn & $\usym{1F5F4}$ & $\usym{1F5F8}$ & $71.56_{\pm0.68}$ (\textcolor{blue}{8.69}) & $2.67_{\pm0.77}$ (\textcolor{blue}{2.67}) & $70.84_{\pm0.60}$ (\textcolor{blue}{4.88}) & $4.67_{\pm3.06}$ (\textcolor{blue}{4.67}) & \textcolor{blue}{5.23} & 27.73 & 12,778,516 (50\%)\\
COUN+RL & $\usym{1F5F4}$ & $\usym{1F5F8}$ & $70.04_{\pm2.01}$ (\textcolor{blue}{10.21}) & $1.36_{\pm0.35}$ (\textcolor{blue}{1.36}) & $69.72_{\pm1.82}$ (\textcolor{blue}{6.00}) & $4.67_{\pm3.06}$ (\textcolor{blue}{4.67}) & \textcolor{blue}{5.56} & 15.24 & 25,557,032 (100\%)\\
DELETE & $\usym{1F5F4}$ & $\usym{1F5F8}$ & $72.46_{\pm1.30}$ (\textcolor{blue}{7.79}) & $0.95_{\pm0.78}$ (\textcolor{blue}{0.95}) & $71.90_{\pm1.23}$ (\textcolor{blue}{\bf{3.82}}) & $2.00_{\pm3.46}$ (\textcolor{blue}{2.00}) & \textcolor{blue}{3.64} & 17.82 & 25,557,032 (100\%)\\
\midrule
{\bf\ours} & only $\presigmarm$ & only $\presigmafg$ & $79.96_{\pm0.38}$ (\textcolor{blue}{\bf0.29}) & $0.00_{\pm0.00}$ (\textcolor{blue}{\bf0.00}) & $69.40_{\pm0.35}$ (\textcolor{blue}{6.32}) & $0.00_{\pm0.00}$ (\textcolor{blue}{\bf0.00}) & \textcolor{blue}{\bf1.65} & {\bf0.94} & {\bf1,024,000} ({\bf4.01\%})\\
\bottomrule
\end{tabular}}
\label{tab:exp-imgnet-resnet50}
\end{table*}

\begin{table*}[t]
\centering
\caption{Comparison results under instance unlearning with {Swin-T} on {Tiny-ImageNet}.}
\resizebox{0.9\textwidth}{!}{
\begin{tabular}{c|cc|ccc|c|ccc}
\toprule
method & $\rmdata$ & $\fgdata$ & $\rm\bf Acc_{rm}^{tr}$ & $\rm\bf Acc_{fg}^{tr}$ & $\rm\bf Acc^{te}$ & {\bf MIA} & Avg.G.$\downarrow$ & RTE $\downarrow$& \# param. (\%)\\
\midrule
$\pref$ & $\usym{1F5F8}$ & $\usym{1F5F8}$ & $99.63_{\pm0.00}$ & $99.80_{\pm0.01}$ & $74.55_{\pm0.00}$ & $96.20_{\pm0.44}$ & - & 882.05 & 27,673,154 (100\%) \\
$\exactf$ & $\usym{1F5F8}$ & $\usym{1F5F4}$ & $99.67_{\pm0.02}$ (\textcolor{blue}{0.00}) & $75.97_{\pm1.00}$ (\textcolor{blue}{0.00}) & $75.03_{\pm0.10}$ (\textcolor{blue}{0.00}) & $64.47_{\pm1.25}$ (\textcolor{blue}{0.00}) & \textcolor{blue}{0.00} & 857.21 & 27,673,154 (100\%) \\
\midrule
GA & $\usym{1F5F4}$ & $\usym{1F5F8}$ & $99.63_{\pm0.00}$ & $99.80_{\pm0.01}$ & $74.55_{\pm0.00}$ & $96.20_{\pm0.44}$ & - & 882.05 & 27,673,154 (100\%) \\
RL & $\usym{1F5F4}$ & $\usym{1F5F8}$ & $97.56_{\pm0.62}$ (\textcolor{blue}{2.11}) & $95.00_{\pm1.64}$ (\textcolor{blue}{19.03}) & $70.57_{\pm0.50}$ (\textcolor{blue}{4.46}) & $88.83_{\pm0.99}$ (\textcolor{blue}{24.36}) & \textcolor{blue}{12.49} & 7.61 & 27,673,154 (100\%)\\
SalUn & $\usym{1F5F4}$ & $\usym{1F5F8}$ & $98.47_{\pm0.34}$ (\textcolor{blue}{\bf1.20}) & $97.23_{\pm1.01}$ (\textcolor{blue}{21.26}) & $71.86_{\pm0.24}$ (\textcolor{blue}{3.17}) & $89.60_{\pm1.08}$ (\textcolor{blue}{25.13}) & \textcolor{blue}{12.69} & 12.79 & 13,836,577 (50\%)\\
L2UL& $\usym{1F5F4}$ & $\usym{1F5F8}$ & $94.77_{\pm0.51}$ (\textcolor{blue}{4.90}) & $74.90_{\pm0.10}$ (\textcolor{blue}{\bf1.07}) & $68.69_{\pm0.27}$ (\textcolor{blue}{6.34}) & $73.90_{\pm0.46}$ (\textcolor{blue}{9.43}) & \textcolor{blue}{5.44} & 7919.11 & 27,673,154 (100\%)\\
COUN+RL& $\usym{1F5F4}$ & $\usym{1F5F8}$ & $98.30_{\pm0.32}$ (\textcolor{blue}{1.37}) & $96.47_{\pm0.21}$ (\textcolor{blue}{20.50}) & $71.45_{\pm0.26}$ (\textcolor{blue}{3.58}) & $89.02_{\pm0.82}$ (\textcolor{blue}{24.55}) & \textcolor{blue}{12.50} & 11.45 & 27,673,154 (100\%)\\
DELETE & $\usym{1F5F4}$ & $\usym{1F5F8}$ & $91.84_{\pm2.12}$ (\textcolor{blue}{7.83}) & $77.57_{\pm2.39}$ (\textcolor{blue}{1.60}) & $65.44_{\pm1.43}$ (\textcolor{blue}{9.59}) & $68.67_{\pm1.50}$ (\textcolor{blue}{4.20}) & \textcolor{blue}{5.81} & 8.42 & 27,673,154 (100\%)\\
\midrule
{\bf\ours} & only $\presigmarm$ & only $\presigmafg$ & $92.65_{\pm0.54}$ (\textcolor{blue}{7.02}) & $82.97_{\pm1.17}$ (\textcolor{blue}{7.00}) & $68.14_{\pm0.42}$ (\textcolor{blue}{6.89}) & {${64.47_{\pm2.36}}$ (\textcolor{blue}{\bf0.00})} & \bf{\textcolor{blue}{5.23}} & \bf{0.50} & \bf{230,400} (0.83\%)\\
\bottomrule
\end{tabular}}
\label{tab:exp-timgnet-swint-instance}
\end{table*}

\subsection{Main Comparisons}
\label{sec:exp:comparison}

\paragraph{Multi/Single-class unlearning.}
Tables~\ref{tab:exp-timgnet-swint-class} and~\ref{tab:exp-timgnet-swint-class-complexity} present comparisons on the unlearning results and computational costs with Swin-T on Tiny-ImageNet, where 4 of 200 classes (2\%) are randomly selected as $\fgdata$.
Table~\ref{tab:exp-timgnet-swint-class} shows that unlearning without access to $\rmdata$ leads to substantial performance drops for baselines such as RL and SalUn, showing their strong reliance to $\rmdata$. 
Our \ours leverages the two feature covariance matrices and achieves highly effective unlearning performance with a smallest average gap of $0.85$ to the retrained model.
On efficiency, Table~\ref{tab:exp-timgnet-swint-class-complexity} shows that the parameter number and the Run Time Efficiency (RTE) in our \ours are less by orders of magnitude than that of other baselines, as \ours simply optimizes a projection matrix operating on two covariance matrices and others generally update the pretrained model with the original training data.
A detailed discussion on the time and memory cost through the entire pipeline of \ours is presented in Appendix \ref{app:sec:loft-efficiency}, where \ours maintains its efficiency advantage.

For single-class unlearning, we consider the much more challenging ImageNet-1K dataset with ResNet50, where 1 of 1,000 classes is randomly selected as $\fgdata$ with 3 repeated runs.
Given the extremely large scale of this dataset ($\rmdata$ contains over 1.28M samples), we therefore compare with methods that only utilize $\fgdata$.
The results in Table~\ref{tab:exp-imgnet-resnet50} demonstrates that \ours outperforms other baselines with the smallest average gap, together with the shortest running time (less than 1 second) and the fewest number of parameters (reduced by approximately 96\%).

\paragraph{Instance unlearning.}
\ours is also applicable to instance-wise unlearning with results shown in Table~\ref{tab:exp-timgnet-swint-instance}, where 1,000 of 100,000 training samples (1\%) in Tiny-ImageNet are randomly selected as $\fgdata$.
Accuracy on the whole test dataset is evaluated as $\rm Acc^{te}$.
Under this random forgetting setting, $\fgdata$ and $\rmdata$ contain samples from all classes, which is more challenging to seek a feature subspace well differentiating $\fgdata$ and $\rmdata$.
Nevertheless, in Table~\ref{tab:exp-timgnet-swint-instance}, the low MIA gap of \ours implies that its learned subspace can effectively remove information related to $\fgdata$.
Furthermore, \ours attains superior performance and efficiency on the overall Avg.G. metric across all baselines by optimizing only <1\% parameters w.r.t. the given model in 0.5 seconds.

\paragraph{More unlearning setups, applications and generation tasks.} 
Due to space limitation, we discuss two practical unlearning scenarios in Appendix~\ref{app:sec:exp-extreme-continual}: {\it extreme unlearning}, which forgets a very high ratio of the training samples, and {\it continual unlearning}, which handles multiple unlearning requests in sequence.
We further consider two real-world applications in Appendix~\ref{app:sec:application}: {\it face recognition}, where the model is required to forget a particular identity, and {\it emotion recognition}, where it is required to forget a particular emotion category.
Our \ours maintains superior unlearning performance across these unlearning setups and applications.
Furthermore, Appendix \ref{app:sec:generation} provides a preliminary exploration of leveraging the low-dimensional feature subspace for unlearning in generation tasks, a still underexplored area, and our initial results demonstrate the potential of \ours in this setting while offering beneficial insights for future research.
Collectively, these results further demonstrate the broad generalizability of our feature-subspace learning approach for machine unlearning across diverse scenarios, real-world applications, and task paradigms.

\subsection{Ablation Study and Sensitivity Analysis}
\label{sec:exp:ablation-sensitivity}

\begin{wraptable}[7]{r}{0.44\textwidth}
\vspace{-\baselineskip}
\centering
\caption{Ablation study on $\objrm$ and $\objfg$.}
\resizebox{0.44\textwidth}{!}{
\begin{tabular}{c|cccc|c}
\toprule
objective & $\rm\bf Acc_{rm}^{tr}$ & $\rm\bf Acc_{fg}^{tr}$ & $\rm\bf Acc_{rm}^{te}$ & $\rm\bf Acc_{fg}^{te}$ & {\bf MIA}\\
\midrule
None & 97.11 & 98.10 & 71.61 & 65.00 & 86.10 \\
$\objfg$ only & 89.51 & 11.50 & 65.40 & 4.00 & 0.00 \\
$\objrm$ only & 99.19 & 87.80 & 75.07 & 53.00 & 58.35 \\
$\objrm+\objfg$ & 98.66 & 1.10 & 74.27 & 0.00 & 0.00 \\
\bottomrule
\end{tabular}}
\label{tab:exp-ablation}
\end{wraptable}
\paragraph{Objective terms $\objrm$ and $\objfg$.}
We explore the roles of the two loss terms in the optimization objective ${J}({\bf U})$ of (\ref{eq:opt-obj}): $\objrm$ preserves performance on $\rmdata$ and $\objfg$ guarantees forgetting on $\fgdata$.
By respectively removing $\objrm$ and $\objfg$ from ${J}({\bf U})$, we report the results in Table \ref{tab:exp-ablation}. 
It shows that optimizing $\bf U$ only with $\objfg$ suppresses $\rm Acc_{fg}$ and MIA, while failing to preserve $\rm Acc_{rm}$.
Employing only $\objrm$ manages to achieve high $\rm Acc_{rm}$, but leaves $\fgdata$ hardly forgotten with high $\rm Acc_{fg}$ and MIA.
Hence, a joint optimization with $\objrm$ and $\objfg$ is crucial for effective MU.
Besides, a sensitivity analysis of a weighted sum for $\objrm$ and $\objfg$ is provided in Appendix \ref{sec:appendix:weighted-loss}.

The ``$\objfg$-only'' results in Table~\ref{tab:exp-ablation} show a clear degradation in $\rm Acc_{rm}$, indicating that the current \ours does not fully address the strict setting where no information from $\rmdata$ is available. Nevertheless, \ours substantially relaxes the remaining-data access requirement: a one-time, gradient-free pass can be performed before unlearning to extract and cache features. Once these features are available, \ours computes the covariances and optimizes $\mathbf U$ without revisiting the raw remaining samples, unlike methods that repeatedly access $\rmdata$ during iterative optimization.

\paragraph{Subspace dimensions $s$.}
The projection matrix $\bf U$ is confined on the Stiefel manifold ${\rm St}(d,s)$, consisting of $s$ orthonormal bases in columns. 
Here, the dimension $s$ determines the percentage of captured feature information in the subspace.
A sensitivity analysis w.r.t. varied values of $s$ on the unlearning performance is provided in Figure \ref{fig:sensitivity} (left).
As illustrated, there is a clear trade-off: a larger $s$ preserves higher $\rm Acc_{rm}$ but sacrifices the ability to forget $\fgdata$ and incurs higher computation, whereas a smaller $s$ effectively forgets $\fgdata$ but compromises $\rm Acc_{rm}$.
From our evaluations, we recommend initializing $s$ with preserving 95\% of the explained variance in $\presigmarm$, as this choice readily provides good empirical performance for different settings and tasks evaluated in this work.

\paragraph{Implementation position for subspace learning.}
\label{sec:implement}
The subspace learning in \ours is primarily applied to the penultimate-layer features $\boldsymbol{z}$, while its implementation can be quite flexible to be inserted into different positions of $\pref$. 
As an example, we consider a ResNet18 model that is structured with a preceding module followed by 4 {\tt layer} modules ({\tt layer1}, {\tt layer2}, {\tt layer3}, and {\tt layer4}) and a linear module\footnote{https://docs.pytorch.org/vision/main/\_modules/torchvision/models/resnet.html\#resnet18}.
By default, \ours is applied to the penultimate-layer features, i.e., features after the {\tt layer4} module.
In experiments, we further apply \ours to earlier {\tt layer} modules ({\tt layer1}, {\tt layer2}, and {\tt layer3}) with results shown in Figure \ref{fig:sensitivity} (right).

\begin{wraptable}[9]{r}{0.4\textwidth}
\vspace{-\baselineskip}
\centering
\caption{Hyper-parameters in different positions for \ours.}
    \resizebox{0.3\textwidth}{!}{
\begin{tabular}{c|ccc}
\toprule
position & $d$ & $s$ & steps \\
\midrule
{\tt layer1} & 4,096 & 500 & 500 \\
{\tt layer2} & 2,048 & 350 & 200\\
{\tt layer3} & 1,024 & 9 & 200 \\
{\tt layer4} & 512   & 9 & 50 \\
\bottomrule
\end{tabular}
}
\label{tab:sensitivity-pos}
\end{wraptable}
As illustrated, \ours maintains competitive performance when applied in earlier positions, compared to its default setting of {\tt layer4}.
It further indicates that the feature covariance matrices in earlier layers contain sufficient information to achieve effective MU from a low-dimensional feature subspace perspective.
However, it can cost more computations, due to the larger feature dimensions and more optimization steps for early-layer implementation.
We provide in Table \ref{tab:sensitivity-pos} the corresponding feature dimension $d$, subspace dimension $s$, and the optimization steps when applying \ours to different {\tt layer} modules.
Considering the trade-off between performance and efficiency, it is recommended to implement \ours in the penultimate layer, as done in main evaluations.

\begin{figure}[t]
    \centering
    \includegraphics[width=0.85\linewidth]{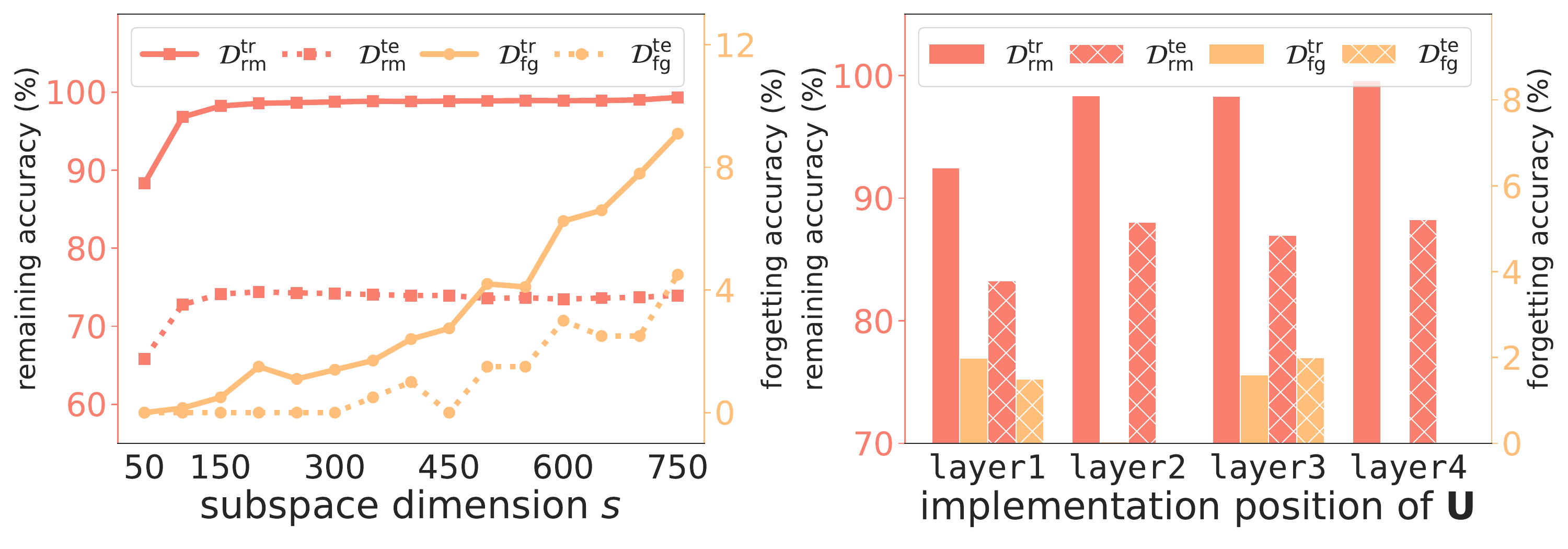}
    \caption{A sensitivity analysis on {\it (i)} subspace dimensions $s$ ({\bf left}) and {\it (ii)} different positions of the projection matrix $\bf U$ ({\bf right}).
    }
    \label{fig:sensitivity}
\end{figure}

\section{Conclusion and Discussion}
This work presents a low-dimensional feature-subspace perspective for MU, a direction that remains largely underexplored yet offers a favorable effectiveness and efficiency trade-off.
Our key insight is that $\rmdata$ and $\fgdata$ can become separable in an appropriate feature subspace, and we aim to learn such a subspace for the pretrained model $\pref$.
Based on this idea, we propose \ours, which leverages PCA-inspired techniques to learn a projection matrix that preserves the knowledge of $\rmdata$ while suppressing that of $\fgdata$.
Its optimization requires only one-shot feature fetching to compute the covariances, avoiding direct and iterative access to raw training data during unlearning.
Moreover, \ours updates only a small projection matrix and can be implemented as a plug-in module to $\pref$ without modifying the full model parameters.
Across extensive evaluations, \ours reduces both the number of updated parameters and the running time by orders of magnitude during unlearning, while maintaining highly competitive unlearning performance.

In addition, since \ours does not update the parameters of $\pref$, its extension to generation tasks is less direct; in such cases, the projection matrix needs to be carefully tailored to the specific generative model and the particular unlearning scenario.
We provide an initial exploration of this direction in Appendix \ref{app:sec:generation}, which suggests the potential of feature-subspace learning for MU in generative tasks.
We hope our explorations could inspire more sophisticated future study.

\begin{ack}
This work was supported by
the National Natural Science Foundation of China (Grant No: 62502416, 62372122, 62503051, and U25A20430), and the Research Grants Council (Grant No:15210023 and 15224124), Hong Kong SAR, China, and the Beijing Natural Science Foundation (L2607017).
\end{ack}

\bibliography{reference}
\bibliographystyle{unsrt}


\appendix

\setcounter{table}{0}
\renewcommand*{\thetable}{S\arabic{table}}
\setcounter{figure}{0}
\renewcommand*{\thefigure}{S\arabic{figure}}
\setcounter{equation}{0}
\renewcommand*{\theequation}{S\arabic{equation}}

Appendices are organized as follows:

$\bullet$ Appendix \ref{app:sec:related-work} provides an overview of related work in MU and introduces the baseline approximated MU methods included in our empirical comparisons.

$\bullet$ Appendix \ref{app:sec:spectrum-analysis} presents more empirical evidences for the empirical study in \cref{sec:low-feat-subspace} to motivate our feature subspace learning.

$\bullet$ Appendix \ref{app:sec:theory-proof} contains theoretical discussions on the convergence of Problem \eqref{eq:opt-obj} and the unlearning performance with complete details of Proposition \ref{prop:loft-gap} in the main text.

$\bullet$ Appendix \ref{app:sec:loft-efficiency} provides a comprehensive analysis into the time and memory cost of \ours throughout its entire pipeline.

$\bullet$ Appendix \ref{app:sec:settings} presents detailed setups of involved MU methods in experiments of {\it multi-class}, {\it single-class} and {\it instance-wise} unlearning in the main text, together with results on the weighted sum of $\objfg$ and $\objrm$.

$\bullet$ Appendix \ref{app:sec:exp-extreme-continual} supplements results under two practical unlearning setups of {\it extreme unlearning} and {\it continual unlearning}.

$\bullet$ Appendix \ref{app:sec:application} supplements results on two real-world MU applications of {\it face recognition} and {\it emotion recognition}.

$\bullet$ Appendix \ref{app:sec:privacy-threat} discusses the privacy scope of \ours and potential threat models.

$\bullet$ Appendix \ref{app:sec:generation} gives a preliminary exploration on the low-dimensional feature subspace learning for MU in generation tasks.

\section{Related Work on Machine Unlearning}
\label{app:sec:related-work}

Machine unlearning aims at maintaining prediction performance of a well-trained model $\pref$ when removing (forgetting) specific data  \cite{cao2015towards,ginart2019making,bourtoule2021machine}. 
In general, such forgetting data $\fgdata$ can mainly be categorized into {\it class-wise} and {\it instance-wise} types.
The former indicates that $\fgdata$ includes all the samples from one or multiple classes in $\cal Y$ \cite{tarun2023fast,chundawat2023zero,zhou2025decoupled} while the latter implies that $\fgdata$ can be a random subset of the entire training data ${\cal D}$ with mixed classes and even all classes \cite{kim2022efficient,fansalun,cha2024learning}.

A golden standard of MU is to retrain the model from scratch on the remaining data $\rmdata$ only, known as {\it the exact MU} \cite{bourtoule2021machine,thudi2022necessity}.
The resulting retrained model $\exactf$ never sees the forgetting data $\fgdata$ and thereby is an oracle in evaluating MU performance.
However, the model retraining on $\rmdata$ can suffer from heavy computational overload, which motivates a series of researches optimizing the pretrained model $\pref$ by approaching to the ideal performance of the exact retrained model $\exactf$, namely {\it the approximate MU}.
Existing approximate MU methods generally take different techniques to update $\pref$, usually requiring visiting the  forgetting data $\fgdata$ and/or the  remaining data $\rmdata$ from the original training dataset.
In the following, the involved MU methods in our comparisons are outlined, covering mainstream unlearning techniques and marking their reliance to $\fgdata$ and $\rmdata$.


{\bf GA} \cite{graves2021amnesiac,thudi2022unrolling} fine-tunes $\pref$ by applying gradient ascent only on the forgetting data $\fgdata$, implemented by maximizing the cross entropy loss.
GA with $\fgdata$ only is highly sensitive to the learning rate and usually leads to a substantial drop on the performance of $\rmdata$.
In our experiments, GA is enhanced with an additional step of executing gradient descent on the remaining data $\rmdata$ \cite{warnecke2021machine,golatkar2020eternal}, and demonstrates nearly perfect performance as the retrained model $\exactf$.

{\bf RL} \cite{golatkar2020eternal} fine-tunes $\pref$ by replacing the labels of $\fgdata$ with new different random labels.
RL is implemented with both $\fgdata$ and $\rmdata$ involved and can easily approach the performance of the ideal retrained model $\exactf$.
Nevertheless, existing research has highlighted the importance of restricted data access for approximate MU, particularly with only $\fgdata$ accessible \cite{cha2024learning,zhou2025decoupled}.
Under this more practical setting with limited data access, RL excluding $\rmdata$ in optimization shows unsatisfactory accuracy on $\rmdata$ and $\fgdata$, verified by our empirical results in this work.

{\bf SalUn} \cite{fansalun} proposes to mask those parameters in $\pref$ that are sensitive to $\fgdata$.
In implementation, a parameter saliency mask is derived based on $\fgdata$ only, and then is applied to the fine-tuning of $\pref$.
The fine-tuning of SalUn follows RL and similarly faces the aforementioned challenge as RL.
SalUn excluding $\rmdata$ in the fine-tuning shows uncompetitive unlearning results and the saliency mask does not bring significant performance gains to RL when  $\fgdata$ is only available.
Another technique with a similar idea of leveraging parameter sparsity is applying an additional $\ell_1$ regularization on the parameters of $\pref$ based on RL \cite{jia2023model}.
Besides SalUn, SFR-on \cite{huang2024unified} further introduces gradient information via Hessians into the parameter mask and proposes a unified optimization framework with a fast-slow weight update.
SFR-on also heavily relies on $\rmdata$ in computing the Fisher information matrix and the loss function.


{\bf L2UL} \cite{cha2024learning} introduces adversarial examples (AEs) \cite{madry2018towards} and weight importance \cite{aljundi2018memory} to promote MU given access only to $\fgdata$.
As claimed in L2UL, AEs w.r.t. $\fgdata$ can mimic the remaining data $\rmdata$, so that minimizing the cross-entropy loss on those AEs can benefit the performance on $\rmdata$.
A parameter importance mask is further applied, where changes of weights that are less important in classifying $\fgdata$ get penalized.
L2UL is particularly superior in instance-wise unlearning.
However, the computational demand of L2UL is prohibitively heavy.
Generating AEs with its default configuration (100 iterations, 200 AEs per image) for a dataset as large as ImageNet-1K (over 1.28M training images) renders the approach impractical.

{\bf COUN} \cite{khalilcoun} leverages the contrastive learning \cite{chen2020simple} for MU.
COUN basically fine-tunes $\pref$ by supervised training only on the remaining data $\rmdata$, and adds an additional contrastive learning loss on $\rmdata$.
The rationale behind is to implicitly push forgetting representations towards those remaining representations with the highest semantic similarity by regularizing the augmented views of $\rmdata$ via the contrastive loss.
As originally reported, COUN is trained solely on $\rmdata$ with improved results. 
To explore whether its contrastive learning paradigm can consistently enhance performance when only $\fgdata$ is available, we adapt COUN to operate with the RL method using only $\fgdata$, denoted as ``COUN+RL'' in our comparisons.
Under this restricted data access, the contrastive learning implemented in COUN brings limited performance gains, as demonstrated in our results.

{\bf DELETE} \cite{zhou2025decoupled} is a distillation-based method relying only on $\fgdata$.
In DELETE, a copy of the pretrained network is utilized as the teacher model, and the key is to apply a mask to the outputs from the teacher.
This mask aims at altering the predictive probabilities of $\fgdata$ to approximate the ideal case and is implemented in a simple way: positions in the mask w.r.t. the ground-truth labels of $\fgdata$ are set to negative infinity and others are set to zero.
There also exist other distillation-based methods, such as bad-teaching \cite{chundawat2023can} and SCRUB \cite{kurmanji2023towards}, which adopt distinct designs to preserve $\rmdata$ and suppress $\fgdata$ during distillation.
Among these, we include the most advanced DELETE into comparisons.

Beyond the aforementioned baselines, we further review several subspace-based MU methods that are more closely relevant with our \ours.
These methods typically define their subspaces in a {\it non-learnable} way based on network {\it parameters}.
Gradient subspaces in convolution neural networks are identified in \cite{li2023subspace,fu2024client} through eigen-decomposition.
In \cite{lizzo2025unlearn}, the subspace in large language models is constructed by applying the Gram-Schmidt process to the singular vectors of parameters.
A latest unlearning method SEMU \cite{sendera2025semu} performs SVD disentanglement on the gradients of $\fgdata$ to identify and update parameters mostly important for unlearning. 
However, our empirical results show that SEMU yields poor unlearning performance on SwinT.
This limitation possibly stems from the fact that its decomposition, applied to parameters of 2d convolution and linear layers, may not adequately adapt to the attention structure in transformers.
Thereby we omit comparisons with SEMU.
In short, our \ours fundamentally differs from these methods by leveraging a {\it learnable feature subspace} for approximate MU.

Researches in MU have been extended to a wide range of fields with different tasks.
The class-wise and concept forgetting in image generation is explored in \cite{fansalun,huang2024unified}.
Privacy attacks against MU methods and corresponding defenses are discussed in \cite{xiao2025reminiscence}.
Unlearning from adversarially trained models is exploited to  achieve superior unlearning performance and also strong adversarial robustness \cite{liu2023muter}.
The concept of unnecessary unlearning is formalized in \cite{li2025funu} with an algorithm to filter those unnecessary requests.
MU can be an effective tool for defending the backdoor attacks \cite{class2liu2022backdoor,liu2024backdoor} through forgetting the backdoor triggers hidden in data.
Moreover, MU in large language models has also received significant attention with extensive researches in recent years \cite{xu2025obliviate,liu2025rethinking,pmlr-v306-xu26cd,xu2026domains}.

\begin{figure*}[t]
    \centering
    \includegraphics[width=0.99\linewidth]{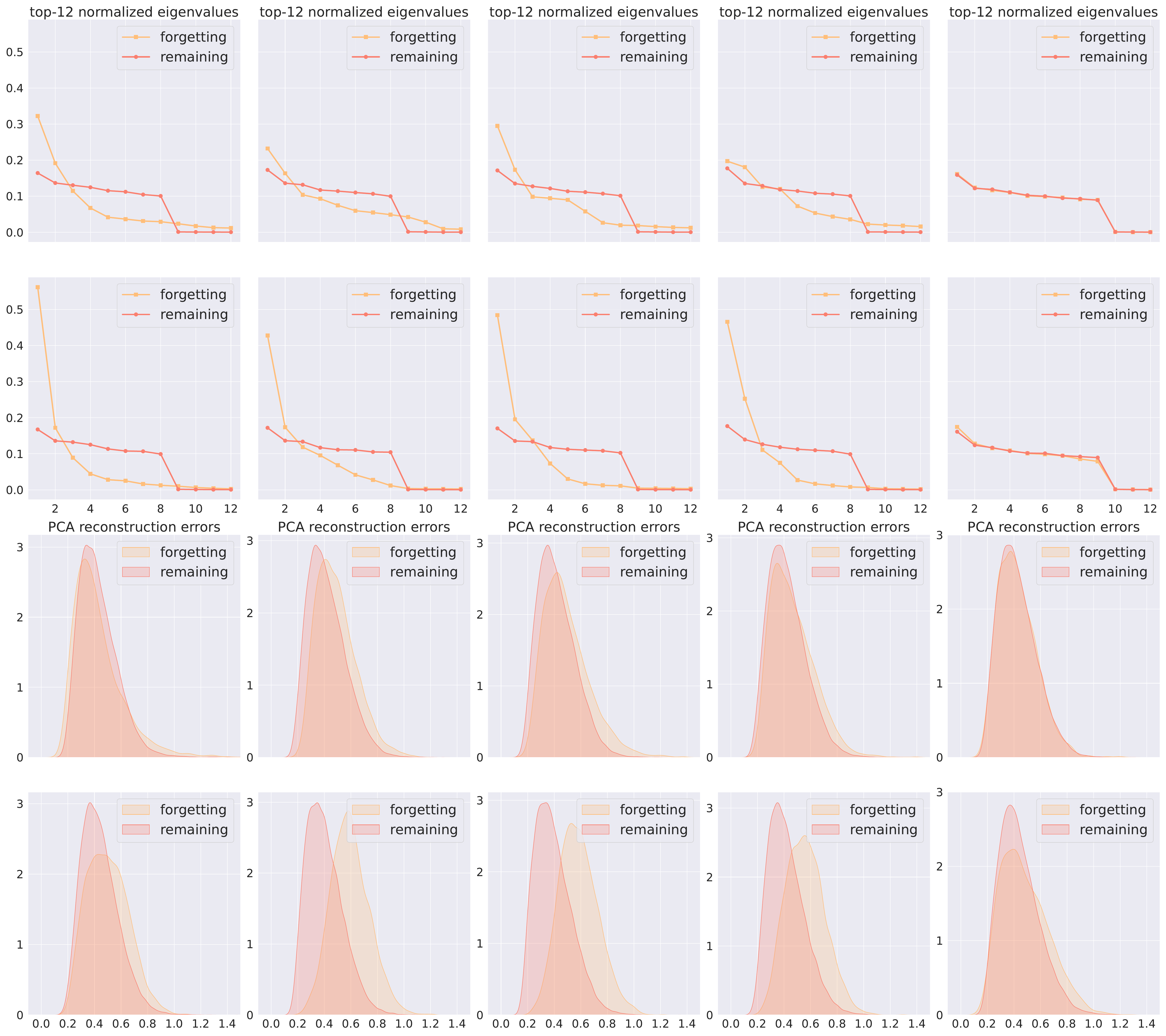}
    \caption{Spectrum and reconstruction analyses on features of $\fgdata$ and $\rmdata$ using the pretrained $\pref$ and the retrained $\exactf$.
    The results are presented as top-12 normalized eigenvalues and reconstruction errors. 
    {\bf Rows 1 \& 2:} Eigenvalues for $\pref$ and $\exactf$, respectively.
    {\bf Rows 3 \& 4:} Reconstruction errors for $\pref$ and $\exactf$, respectively.
    {\bf Columns 1-4:} $\fgdata$ including images from labels 1,3,5,7, respectively.
    {\bf Column 5:} $\fgdata$ comprising images randomly sampled from all classes.
    }
    \label{fig:spectrum-all}
\end{figure*}

\section{More Evidences for the Empirical Study in \cref{sec:low-feat-subspace}}
\label{app:sec:spectrum-analysis}

The empirical study in \cref{sec:low-feat-subspace} on the spectrum and reconstruction analyses motivates our low-dimensional feature subspace perspective for approximate MU.
In the eigenvalue and reconstruction error results presented in \cref{fig:spectrum}, we adopt a class-wise forgetting setting, where images corresponding to label-0 in CIFAR10 are selected as the forgetting data $\fgdata$.
To further strengthen this empirical study, additional evidences under different forgetting-remaining splits are provided in this section.

In addition to $\fgdata$ containing label-0 in \cref{fig:spectrum}, we consider 4 other classes in CIFAR10, labels-1,3,5,7, as $\fgdata$, and execute the same spectrum and reconstruction analyses described in \cref{sec:low-feat-subspace}.
The corresponding results are shown in the first 4 columns of \cref{fig:spectrum-all}.
Besides, we further extend our analyses to the more challenging instance-wise unlearning setting, where $\fgdata$ consists of images randomly-sampled across all classes; these results are presented in the last column of \cref{fig:spectrum-all}.

As illustrated in \cref{fig:spectrum-all}, for additional class-wise forgetting cases, the separability between $\fgdata$ and $\rmdata$ under a low-dimensional subspace from $\exactf$ remains observable through the disparities in eigenvalues and the gaps in reconstruction errors.
In contrast, for instance-wise unlearning, where $\fgdata$ and $\rmdata$ are highly similar and nearly drawn from the same distribution, the retrained model $\exactf$ effectively learns that distribution from $\rmdata$, resulting in nearly identical eigenvalues and largely overlapped reconstruction errors.
Thus, the separability achieved by directly applying PCA to features from $\exactf$ is considerably weaker in instance‑wise unlearning than in class‑wise forgetting.

As we have remarked at the end of \cref{sec:low-feat-subspace}, the purpose of this analysis is {\it not to claim universal separability across all unlearning scenarios}, but rather to motivate our feature subspace learning perspective. 
For class-wise unlearning, the natural separability is evident, and we seek to approximate this ideal subspace from the pretrained $\pref$.

For instance-wise unlearning, however, such separability cannot be expected to emerge naturally via unsupervised non-parametric methods like PCA, as $\fgdata$ and $\rmdata$ are nearly drawn from the same underlying distribution. 
{\it This highlights exactly why a dedicated subspace learning objective is necessary.} 
Instead of relying on natural emergence, our proposed method explicitly induces this separability. 
By actively minimizing the projected variance of the specific $\fgdata$ while preserving the structure of $\rmdata$ (as formulated in our objective of \eqref{eq:opt-obj}), our method forces the projection matrix to discover a specific low-dimensional direction where $\rmdata$ and $\fgdata$ manifest separability. 
This explains why our approach remains effective for instance-wise unlearning in \cref{sec:exp}, despite the absence of natural PCA separability.

\section{Theoretical Discussions}
\label{app:sec:theory-proof}

\subsection{Theoretical Discussions on Convergence}
\label{app:sec:convergence}

In this section, we show that $J({\bf U})$ is $L$-smooth and thereby the convergence rate of Riemannian Adam on minimizing $J({\bf U})$ is guaranteed.

\paragraph{$J({\bf U})$ is a polynomial.}
Recall that 
$$
{J}({\bf U})
=
\underbrace{
\Bigl(
\frac
{{\rm Tr}({\bf U}^\top\presigmafg{\bf U})}
{{\rm Tr}(\presigmafg)}
\Bigr)^2}_{\objfg}
+
\underbrace{
\Bigl(
\frac
{{\rm Tr}\left(\presigmarm-{\bf U}{\bf U}^\top\presigmarm{\bf U}{\bf U}^\top\right)}
{{\rm Tr}\left(\presigmarm\right)}
\Bigr)^2}_{\objrm}.
$$
We can find that both terms of $J({\bf U})$ are polynomials in the entries of $\bf U$:
\begin{itemize}[leftmargin=20pt]
    \item $\objfg=\frac{1}{c_1^2}\left({\rm Tr}({\bf U}^\top\presigmafg{\bf U})\right)^2$ is a polynomial of degree $\leq4$ with $c_1={\rm Tr}(\presigmafg)$.
    \item $\objrm=\frac{1}{c_2^2}(c_2-{\rm Tr}({\bf U}^\top\presigmarm{\bf U}))^2$ is a polynomial of degree $\leq8$ with $c_2={\rm Tr}(\presigmarm)$.
\end{itemize}
Hence $J({\bf U})=\objfg+\objrm$ is a polynomial (degree $\leq8$) with ${\bf U}\in\mathbb{R}^{d\times s}$.
Consequently, $J({\bf U})$ is $C^\infty$ and all its partial derivatives are polynomials.

\paragraph{Euclidean Gradient is Lipschitz on the Compact Stiefel Manifold.}
Let $\nabla J(\mathbf{U})$ denote the Euclidean (standard) gradient in $\mathbb{R}^{d \times s}$. Because $J$ is a polynomial of degree $\leq 8$, $\nabla J(\mathbf{U})$ is a polynomial of degree $\leq 7$, and its derivative $\nabla^{2} J(\mathbf{U})$ (the Euclidean Hessian) is a polynomial of degree $\leq 6$.

The Stiefel manifold $\operatorname{St}(d,s)$ is compact (closed and bounded in $\mathbb{R}^{d \times s}$). Therefore the continuous function $\mathbf{U} \mapsto \| \nabla^{2} J(\mathbf{U}) \|_{\mathrm{op}}$ attains a maximum on this compact set. Define
\[
L_{E}\coloneqq \max_{\mathbf{U}\in \operatorname{St}(d,s)}\| \nabla^{2} J(\mathbf{U}) \|_{\mathrm{op}} < \infty .
\]
By the mean value inequality (or Taylor's theorem with remainder), for any $\mathbf{U},\mathbf{V}\in \operatorname{St}(d,s)$,
\[
\| \nabla J(\mathbf{U}) - \nabla J(\mathbf{V})\|_{F}\leq L_{E}\| \mathbf{U} - \mathbf{V}\|_{F}.
\]
Thus the Euclidean gradient is $L_{E}$-Lipschitz on the Stiefel manifold with respect to the Frobenius norm.

\paragraph{Riemannian Gradient and Its Lipschitz Property.}
For a submanifold of $\mathbb{R}^{d\times s}$, the Riemannian gradient $\operatorname{grad}J(\mathbf{U})$ is the orthogonal projection of the Euclidean gradient onto the tangent space:
\[
\operatorname{grad}J(\mathbf{U}) = \Pi_{\mathbf{U}}(\nabla J(\mathbf{U})),
\]
where $\Pi_{\mathbf{U}}:\mathbb{R}^{d\times s}\to T_{\mathbf{U}}\mathrm{St}(d,s)$ is the orthogonal projector. $\Pi_{\mathbf{U}}$ depends smoothly on $\mathbf{U}$ and satisfies $\| \Pi_{\mathbf{U}}\|_{\mathrm{op}} = 1$.

We need to bound the difference between Riemannian gradients after parallel transport. Let $\tau_{\mathbf{U}\to \mathbf{V}}$ denote parallel transport along a minimizing geodesic from $\mathbf{U}$ to $\mathbf{V}$. A standard result in Riemannian optimization (see e.g. \cite{absil2008optimization}, Section 4.3) states:
\begin{lemma}
\label{lm:gradbound}
If a function $f$ on a compact Riemannian manifold is $C^2$ and its Hessian is uniformly bounded, then its gradient is Lipschitz continuous. More precisely, there exists a constant $L_{f} > 0$ such that for all $\mathbf{U},\mathbf{V}$ in the manifold,
\[
\| \operatorname{grad}f(\mathbf{U}) - \tau_{\mathbf{V}\to \mathbf{U}}\operatorname{grad}f(\mathbf{V})\| \leq L_{f}\operatorname{dist}(\mathbf{U},\mathbf{V}),
\]
where $\operatorname{dist}$ denotes the Riemannian distance.
\end{lemma}

\begin{lemma}
$J(\mathbf{U})$ is $L$-smooth.
\end{lemma}
\begin{proof}
Because our $J(\mathbf{U})$ is $C^{\infty}$ (hence $C^2$) and $\operatorname{St}(d,s)$ is compact, the Riemannian Hessian $\operatorname{Hess}J(\mathbf{U})$ is uniformly bounded in operator norm. Therefore Lemma \ref{lm:gradbound} applies, and we conclude that there exists a constant $L > 0$ such that
\[
\| \operatorname{grad}J(\mathbf{U}) - \tau_{\mathbf{V}\to \mathbf{U}}\operatorname{grad}J(\mathbf{V})\| \leq L\operatorname{dist}(\mathbf{U},\mathbf{V}),\qquad \forall \mathbf{U},\mathbf{V}\in \operatorname{St}(d,s).
\]

For completeness, we provide an elementary bound without invoking parallel transport directly, using the Euclidean Lipschitz property and the smoothness of the projector.

For any $\mathbf{U},\mathbf{V}\in \operatorname{St}(d,s)$,
\[
\| \operatorname{grad}J(\mathbf{U}) - \operatorname{grad}J(\mathbf{V})\|_{F} = \| \Pi_{\mathbf{U}}\nabla J(\mathbf{U}) - \Pi_{\mathbf{V}}\nabla J(\mathbf{V})\|_{F}
\]
\[
\leq \| \Pi_{\mathbf{U}}(\nabla J(\mathbf{U}) - \nabla J(\mathbf{V}))\|_{F} + \| (\Pi_{\mathbf{U}} - \Pi_{\mathbf{V}})\nabla J(\mathbf{V})\|_{F}.
\]

Using the Euclidean Lipschitz bound,
\[
\| \nabla J(\mathbf{U}) - \nabla J(\mathbf{V})\|_{F}\leq L_{E}\| \mathbf{U} - \mathbf{V}\|_{F}.
\]

For the second term, note that the map $\mathbf{U}\mapsto \Pi_{\mathbf{U}}$ is smooth on the compact Stiefel manifold, hence Lipschitz. Thus there exists a constant $C_{\Pi} > 0$ such that
\[
\| \Pi_{\mathbf{U}} - \Pi_{\mathbf{V}}\|_{\mathrm{op}}\leq C_{\Pi}\| \mathbf{U} - \mathbf{V}\|_{F}.
\]

Moreover, $\| \nabla J(\mathbf{V})\|_{F}$ is bounded on the compact set $\operatorname{St}(d,s)$; denote $G = \max_{\mathbf{V}\in \operatorname{St}(d,s)}\| \nabla J(\mathbf{V})\|_{F}< \infty$. Therefore,
\[
\| (\Pi_{\mathbf{U}} - \Pi_{\mathbf{V}})\nabla J(\mathbf{V})\|_{F}\leq C_{\Pi}G\| \mathbf{U} - \mathbf{V}\|_{F}.
\]

Combining both estimates, we obtain
\[
\| \operatorname{grad}J(\mathbf{U}) - \operatorname{grad}J(\mathbf{V})\|_{F}\leq (L_{E} + C_{\Pi}G)\| \mathbf{U} - \mathbf{V}\|_{F}.
\]

On a compact Riemannian manifold, the Euclidean distance $\| \mathbf{U} - \mathbf{V}\|_{F}$ and the Riemannian distance $\operatorname{dist}(\mathbf{U},\mathbf{V})$ are equivalent (the ratio is bounded by the maximum of the inverse of the injectivity radius). Hence there exists a constant $L^{\prime} > 0$ such that
\[
\| \operatorname{grad}J(\mathbf{U}) - \operatorname{grad}J(\mathbf{V})\|_{F}\leq L^{\prime}\operatorname{dist}(\mathbf{U},\mathbf{V}).
\]

This shows the Riemannian gradient is Lipschitz without explicitly using parallel transport. The Lipschitz constant $L$ in the parallel-transport sense can be taken as the same $L^{\prime}$ because parallel transport is an isometry and the gradient vectors are identified appropriately.
\end{proof}

As $J({\bf U})$ is smooth, by existing results in \cite{wang2024provable}, we can draw the following conclusion:
\begin{lemma}[Theorem 4.1 in \cite{wang2024provable}]
\label{lm:convergence}
If $J(\mathbf{U})$ is $L$-smooth, by using Riemannian Adam algorithm,
\[
\min_{t\in [T]}\| \nabla J(\mathbf{U}_{t})\| \leq O\!\left(\frac{1}{\sqrt{T}} + (1 - \beta_{2})^{2}\right).
\]
This guarantees the convergence rate of Riemannian Adam on minimizing $J(\mathbf{U})$.
\end{lemma}
Here, $T$ denotes the number of optimization steps and $\beta_2$ is the exponential decay rate for the second moment estimate (the uncentered variance) in the Riemannian Adam optimizer with a typical default value of $0.999$.
Clearly, the gradients of $J(\mathbf{U})$ approximate 0 given a large step $T$, and thereby the convergence is guaranteed.

\subsection{Theoretical Discussions on Unlearning}
\label{app:sec:unlearning}
A detailed version of Proposition \ref{prop:loft-gap} in \cref{sec:method} is stated as follows.
\begin{proposition}
For the pretrained model $\pref(\cdot)=\preh(\preg(\cdot))$ and the retrained model $\exactf(\cdot)=\exacth(\exactg(\cdot))$, let their linear classifiers correspondingly be $\preh(\boldsymbol{z})={\bf W}_{\rm pre}^\top\boldsymbol{z}$ and $\exacth(\boldsymbol{z})={\bf W}_{\rm exact}^\top\boldsymbol{z}$ (bias omitted for simplicity).
Assuming the following conditions hold:\\
$\bullet$
The learned projection matrix ${\bf\hat U}\in{\rm St}(d,s)$ obtained by solving \eqref{eq:opt-obj} satisfies
$$
{\rm Tr}({\bf\hat U}^\top\presigmafg{\bf\hat U})\leq\epsilon_{\rm fg},\qquad{\rm Tr}(({\bf I}-{\bf\hat U}{\bf\hat U}^\top)\presigmarm({\bf I}-{\bf\hat U}{\bf\hat U}^\top))\leq\epsilon_{\rm rm}
$$
$\bullet$
There exists $\delta_{\rm rm}>0$ such that for all $\boldsymbol{x}\in\rmdata$, $\|\preg(\boldsymbol{x})-\exactg(\boldsymbol{x})\|_2\leq\delta_{\rm rm}$\\
$\bullet$
There exists constants $M_{\rm rm}^{\rm exact}$ and $M_{\rm fg}^{\rm exact}>0$ such that
$$
\sup_{\boldsymbol{x}\in\rmdata}\|\exactg(\boldsymbol{x})\|_2\leq M_{\rm rm}^{\rm exact},\quad
\sup_{\boldsymbol{x}\in\fgdata}\|\exactg(\boldsymbol{x})\|_2\leq M_{\rm fg}^{\rm exact}.
$$
$\bullet$
The linear classifiers satisfy $\|{\bf W}_{\rm pre}-{\bf W}_{\rm exact}\|_2\leq\delta_W$.\\
$\bullet$
For the pretrained $\pref$, features of $\rmdata$ and $\fgdata$ are centered, i.e., $$\mathbb{E}_{\boldsymbol{x}\sim\rmdata}[\preg(\boldsymbol{x})]={\bf 0},\ \mathbb{E}_{\boldsymbol{x}\sim\fgdata}[\preg(\boldsymbol{x})]={\bf 0},
$$
Then, for any $\boldsymbol{x}\in{\cal D}$, the expected output difference between the unlearned model $f_{\bf\hat U}$ and $\exactf$ satisfies:
$$
\mathbb{E}_{\boldsymbol{x}\in\rmdata}
\left[\|
f_{\bf\hat U}(\boldsymbol{x})-\exactf(\boldsymbol{x})
\|_2\right]\leq
\|{\bf W}_{\rm pre}\|_2(\delta_{\rm rm}+\sqrt{\epsilon_{\rm rm}})+\delta_WM_{\rm rm}^{\rm exact}
$$
$$
\mathbb{E}_{\boldsymbol{x}\in\fgdata}
\left[\|
f_{\bf\hat U}(\boldsymbol{x})-\exactf(\boldsymbol{x})
\|_2\right]\leq
\|{\bf W}_{\rm pre}\|_2\sqrt{\epsilon_{\rm fg}}+
\|{\bf W}_{\rm exact}\|_2M_{\rm fg}^{\rm exact}
$$
\end{proposition}

\begin{proof}
We treat the two cases separately.

For $\boldsymbol{x}\in\rmdata$, we have:
\begin{equation}
\begin{aligned}
\label{thm:eq:sun-gap-1}
\|f_{\bf\hat U}(\boldsymbol{x})-\exactf(\boldsymbol{x})\|_2
&=\|{\bf W}_{\rm pre}^\top{\bf\hat U}{\bf\hat U}^\top\preg(\boldsymbol{x})
-{\bf W}_{\rm exact}^\top\exactg(\boldsymbol{x})\|_2\\
&\leq\|{\bf W}_{\rm pre}^\top{\bf\hat U}{\bf\hat U}^\top\preg(\boldsymbol{x})
-{\bf W}_{\rm pre}^\top\preg(\boldsymbol{x})\|_2\\
&+\|{\bf W}_{\rm pre}^\top\preg(\boldsymbol{x})
-{\bf W}_{\rm pre}^\top\exactg(\boldsymbol{x})\|_2\\
&+\|({\bf W}_{\rm pre}-{\bf W}_{\rm exact})^\top\exactg(\boldsymbol{x})\|_2.
\end{aligned}
\end{equation}

The first term is bounded by $\|{\bf W}_{\rm pre}\|_2\|({\bf I}-{\bf\hat U}{\bf\hat U}^\top)\preg(\boldsymbol{x})\|_2$.
By the assumption on the solution of \ours, we have $\mathbb{E}_{\boldsymbol{x}\in\rmdata}\|{\bf W}_{\rm pre}\|_2\|({\bf I}-{\bf\hat U}{\bf\hat U}^\top)\preg(\boldsymbol{x})\|_2\leq\|{\bf W}_{\rm pre}\|_2\sqrt{\epsilon_{\rm rm}}$.
The second term is bounded by $\|{\bf W}_{\rm pre}\|_2\delta_{\rm rm}$.
The third term is bounded by $\|{\bf W}_{\rm pre}-{\bf W}_{\rm exact}\|_2\|\exactg(\boldsymbol{x})\|_2\leq\delta_WM_{\rm rm}^{\rm exact}$.
Summing these gives the claimed bound for $\rmdata$.

For $\boldsymbol{x}\in\fgdata$, similarly, we have:
\begin{equation}
\begin{aligned}
\|f_{\bf\hat U}(\boldsymbol{x})-\exactf(\boldsymbol{x})\|_2
&\leq\|{\bf W}_{\rm pre}^\top{\bf\hat U}{\bf\hat U}^\top\preg(\boldsymbol{x})\|_2
+\|{\bf W}_{\rm exact}^\top\exactg(\boldsymbol{x})\|_2\\
&\leq\|{\bf W}_{\rm pre}\|_2\|{\bf\hat U}{\bf\hat U}^\top\preg(\boldsymbol{x})\|_2+\|{\bf W}_{\rm exact}\|_2\|\exactg(\boldsymbol{x})\|_2.
\end{aligned}
\end{equation}

By the assumption on the solution of \ours, the first term is bounded by $\|{\bf W}_{\rm pre}\|_2\sqrt{\epsilon_{\rm fg}}$ in expectation.
The second term is bounded by $\|{\bf W}_{\rm exact}\|_2M_{\rm fg}^{\rm exact}$.
Therefore the proof is finished.
\end{proof}

Proposition \ref{prop:loft-gap} establishes that the expected output discrepancy between LOFT and the exact unlearned model is explicitly controlled by the two subspace quality measures, $\epsilon_{\rm fg}$ and $\epsilon_{\rm rm}$, which \ours directly minimizes, and other proximity assumptions.
If \ours succeeds in making $\epsilon_{\rm fg}$ and $\epsilon_{\rm rm}$ small, and if the pretrained and retrained models are close on the remaining data, then the unlearned model will be close to the gold-standard $\exactf$.

{\bf Remark.} As we have noted in the end of \cref{sec:method}, \ours does not update the parameters of the pretrained model $\pref$, and thereby cannot be adapted to the existing certified unlearning framework that certifies the parameter proximity between $\pref$ and the unlearned model.
Proposition \ref{prop:loft-gap} is a conditional analysis to support the approximation behavior of \ours, rather than a guarantee of certified unlearning.
A more rigorous theoretical development for subspace approximate MU remains a valuable direction for future work.

\section{In-depth Efficiency Analysis into \ours in Time and Memory}
\label{app:sec:loft-efficiency}
We provide a comprehensive analysis on the time and memory cost of \ours throughout its entire pipeline.
As demonstrated in Algorithm \ref{alg:loft}, \ours can be partitioned into two stages, {\it offline one-shot feature fetching} and {\it on-demand unlearning}.
In the offline stage, \ours performs a full pass over the training dataset using the pretrained network $\pref$, in order to extract the remaining and forgetting features ($\prefeatrm$ and $\prefeatfg$).
Then, in the unlearning stage, given a forgetting request, \ours computes the corresponding covariance matrices ($\presigmarm$ and $\presigmafg$), and executes its approximate unlearning to learn a projection matrix ${\bf U}\in{\rm St}(d\times s)$ by solving \eqref{eq:opt-obj} based on $\presigmarm$ and $\presigmafg$.

Despite iterating through the entire training dataset for feature extraction, \ours remains highly efficient in its offline stage, as summarized in Table \ref{tab:exp-timgnet-swint-class-complexity-fullpipeline}.
For this evaluation, we adopt the same setup as Tables \ref{tab:exp-timgnet-swint-class} and \ref{tab:exp-timgnet-swint-class-complexity} in Section \ref{sec:exp:comparison}, and employ Tiny-ImageNet with Swin-T, where 4 of 200 classes are randomly selected as the forgetting data $\fgdata$.
As demonstrated in Table \ref{tab:exp-timgnet-swint-class-complexity-fullpipeline}, feature extraction is performed via a single inference-only forward pass through the pretrained $\pref$ without backward gradient computations, costing only 11.82 seconds on a GPU.
Subsequently, the two covariance matrices (each 768$\times$768) are computed from the resulting feature matrices (98,000$\times$768 of $\prefeatrm$ and 2,000$\times$768 of $\prefeatfg$), a process that finishes in mere 1.05 seconds on CPU.
Compared with the time consumption of other baselines in Table \ref{tab:exp-timgnet-swint-class-complexity}, the entire pipeline of \ours still retains a clear efficiency advantage.

Regarding the memory cost, the peak overhead of \ours arises from storing the intermediate training features (100,000$\times$768) during its preparation stage, which nevertheless remains minimal at approximately 300 MB.
Aside from that, the unlearning stage works on only two feature covariance matrices of 768$\times$768, and learns a projection matrix of 768$\times$250 without updating the parameters of the pretrained network $\pref$.
As a result, the overall memory footprint of \ours remains nearly negligible.

\paragraph{Remark.}
{\bf \ours is highly efficient particularly for large-scale data.}
During its unlearning stage, \ours learns $\bf U$ based on the two fixed pre-computed feature covariance matrices $\presigmarm$ and $\presigmafg$.
This discards the need for iteratively visiting the raw data or updating the pretrained network $\pref$.
Meanwhile, the feature fetching stage is thereby a one-time cost by processing the training data in a single, gradient-free forward pass to yield $\presigmarm$ and $\presigmafg$.
We think such computational complexity of \ours is acceptable for deep learning in practice and makes \ours particularly suitable for large-scale data scenarios.
In contrast, mainstream unlearning methods face significant scalability bottlenecks: they either {\it (i)} require {\it repeated} forward-{\it backward} passes on full data with gradients calculations, which is prohibitively expensive for large-scale data, or {\it (ii)} rely solely on the forgetting data for efficiency, but leading to limited performance due to the absence of the remaining data information. 
For example, the ImageNet-1K dataset includes over 1.28M training images.
Existing methods that calculates gradients on both forgetting and remaining data with multiple epochs face nearly prohibitive calculations on ImageNet-1K, while \ours can be efficiently and effectively executed by iterating the ImageNet-1K training dataset only once.
Therefore, \ours strikes a favorable balance and achieves effective unlearning by leveraging the forgetting and remaining feature covariance matrices through a single, computationally-manageable forward pass, thereby offering remarkable scalability in practical applications.

\begin{table*}[t]
    \centering
    \caption{A detailed analysis on the time cost of the whole pipeline of \ours.
    Empirical setups follow the same as that of Tables \ref{tab:exp-timgnet-swint-class} and \ref{tab:exp-timgnet-swint-class-complexity}.}
    \begin{tabular}{c|c|cc}
    \toprule
    \multirow{2}{*}{method} & \multirow{2}{*}{offline one-shot feature fetching} & \multicolumn{2}{c}{on-demand unlearning}\\
    & & calculating covariances & optimizing $\bf U$ \\
    \midrule
    {\bf\ours} & 11.82 & 1.05 & 0.24 \\
    \bottomrule
    
    \end{tabular}
    \label{tab:exp-timgnet-swint-class-complexity-fullpipeline}
\end{table*}

\section{Implementation Details of Experiments in Main Text}
\label{app:sec:settings}

\subsection{Training Details of Multi-Class Unlearning}
\label{sec:settings-class}

For multi-class unlearning in Table~\ref{tab:exp-timgnet-swint-class} with {Swin-T} on {Tiny-ImageNet}, we adopt the PyTorch-released checkpoint trained on ImageNet-1K as a starting point, and fine-tune this checkpoint on Tiny-ImageNet to obtain the pretrained model.
In the repeated 3 experiments of Table~\ref{tab:exp-timgnet-swint-class}, the randomly selected 4 forgetting labels (classes) are $\{11,83,115,153\}$, $\{44,65,150,168\}$ and $\{53,57,108,179\}$, respectively. All compared methods use the AdamW optimizer \cite{loshchilovdecoupled} with a weight decay of 0.05 and a batch size of 128.
In L2UL \cite{cha2024learning}, to generate adversarial examples, the $\ell_2$-PGD targeted attack is employed with a step size of 0.1, a perturbation bound of 0.4 and 100 iteration steps, and 200 adversarial examples are generated per image. 
The termination of the L2UL unlearning is carefully selected until the accuracy on the training forgetting data is sufficiently low without a substantial accuracy drop on the training remaining data.
All the training hyper-parameters to obtain results in Table~\ref{tab:exp-timgnet-swint-class} are listed in Table~\ref{tab:exp-timgnet-swint-class-hyperparam}.
Our \ours adopts the Riemannian Adam optimizer \cite{kochurov2020geoopt}  for the projection matrix $\mathbf{U}$ with a weight decay of 0.05, where the penultimate layer feature dimension with Swin-T on Tiny-ImageNet is $d=768$.

\begin{table*}[t]
    \centering
    \caption{Training hyper-parameters of different MU methods w.r.t. results of multi-class unlearning in Table~\ref{tab:exp-timgnet-swint-class} with Swin-T on Tiny-ImageNet.}
    \resizebox{0.85\textwidth}{!}{
    \begin{tabular}{c|cc|c}
    \toprule
    method & $\rmdata$ & $\fgdata$ & hyper-parameters\\
    \midrule
    $\pref$ & $\usym{1F5F8}$ & $\usym{1F5F8}$ & 20 epochs, lr $\rm ={10}^{-4}$, cosine scheduler \\
    $\exactf$ & $\usym{1F5F8}$ & $\usym{1F5F4}$ & 20 epochs, lr $\rm ={10}^{-4}$, cosine scheduler\\
    \midrule
    GA & $\usym{1F5F4}$ & $\usym{1F5F8}$ & 10 epochs, lr $\rm =2\times{10}^{-6}$, constant scheduler\\
    GA & $\usym{1F5F8}$ & $\usym{1F5F8}$ & 5 epochs, lr $\rm ={10}^{-5}$, constant scheduler\\
    RL & $\usym{1F5F4}$ & $\usym{1F5F8}$ & 10 epochs, lr $\rm ={10}^{-5}$, cosine scheduler\\
    RL & $\usym{1F5F8}$ & $\usym{1F5F8}$ & 10 epochs, lr $\rm ={10}^{-4}$, cosine scheduler\\
    SalUn & $\usym{1F5F4}$ & $\usym{1F5F8}$ & 10 epochs, lr $\rm ={10}^{-5}$, cosine scheduler, saliency sparsity 50\%\\
    SalUn & $\usym{1F5F8}$ & $\usym{1F5F8}$ & 10 epochs, lr $\rm ={10}^{-4}$, cosine scheduler, saliency sparsity 50\%\\
    COUN+RL & $\usym{1F5F4}$ & $\usym{1F5F8}$ & 10 epochs, lr $\rm =2\times{10}^{-5}$, cosine scheduler \\
    COUN+RL & $\usym{1F5F8}$ & $\usym{1F5F8}$ & 10 epochs, lr $\rm =2\times{10}^{-4}$, cosine scheduler \\
    L2UL & $\usym{1F5F4}$ & $\usym{1F5F8}$ & lr $\rm ={10}^{-5}$, constant scheduler, regularization coefficient = 1.0\\
    DELETE & $\usym{1F5F4}$ & $\usym{1F5F8}$ & 10 epochs, lr $\rm ={10}^{-5}$, cosine scheduler\\
    \midrule
    {\bf\ours} & only $\presigmarm$ & only $\presigmafg$ & 50 steps, $s=250$, lr $\rm = 1$, constant scheduler\\
    \bottomrule
    
    \end{tabular}}
    \label{tab:exp-timgnet-swint-class-hyperparam}
\end{table*}

\subsection{Training Details of Single-Class Unlearning}
For single-class unlearning in Table~\ref{tab:exp-imgnet-resnet50} with {ResNet50} on {ImageNet-1K}, we deploy the PyTorch-released checkpoint that is exactly pretrained on ImageNet-1K as the pretrained model $\pref$.
Regarding the retrained model $\exactf$ for unlearning, we train ResNet50 from scratch on the remaining data $\rmdata$ for 90 epochs on 4 NVIDIA GeForce RTX 4090 GPUs, which runs for around 24 hours.
In the repeated 3 experiments of Table~\ref{tab:exp-imgnet-resnet50}, the randomly selected forgetting labels are 97, 316, and 852, respectively.
All the compared methods use the SGD optimizer \cite{bottou2012stochastic} with a weight decay of $1\times10^{-4}$, the momentum of 0.9, and a batch size of 128.
All the training hyper-parameters for the results in Table~\ref{tab:exp-imgnet-resnet50} are listed  Table~\ref{tab:exp-imgnet-resnet50-hyperparam}.
Our \ours keeps the settings of using the Riemannian Adam optimizer \cite{kochurov2020geoopt} for $\mathbf{U}$ with a weight decay of 0.05, where the penultimate layer feature dimension with ResNet50 on ImageNet-1K is $d=2048$.

\begin{table*}[t]
    \centering
    \caption{Training hyper-parameters of different MU methods w.r.t. results of single-class unlearning in Table~\ref{tab:exp-imgnet-resnet50} with ResNet50 on ImageNet-1K.}
    \resizebox{0.85\textwidth}{!}{
    \begin{tabular}{c|cc|c}
    \toprule
    method & $\rmdata$ & $\fgdata$ & hyper-parameters\\
    \midrule
    $\pref$ & $\usym{1F5F8}$ & $\usym{1F5F8}$ & - \\
    $\exactf$ & $\usym{1F5F8}$ & $\usym{1F5F4}$ & 90 epochs, lr $\rm ={10}^{-1}$, LR decay: 0.1 every 30 epochs\\
    \midrule
    GA & $\usym{1F5F4}$ & $\usym{1F5F8}$ & 3 epochs, lr $\rm ={10}^{-4}$, constant scheduler\\
    RL & $\usym{1F5F4}$ & $\usym{1F5F8}$ & 5 epochs, lr $\rm =5\times{10}^{-5}$, cosine scheduler\\
    SalUn & $\usym{1F5F4}$ & $\usym{1F5F8}$ & 10 epochs, lr $\rm ={10}^{-4}$, cosine scheduler, saliency sparsity 50\%\\
    COUN+RL & $\usym{1F5F4}$ & $\usym{1F5F8}$ & 3 epochs, lr $\rm ={10}^{-6}$, cosine scheduler \\
    DELETE & $\usym{1F5F4}$ & $\usym{1F5F8}$ & 5 epochs, lr $\rm ={10}^{-1}$, cosine scheduler\\
    \midrule
    {\bf\ours} & only $\presigmarm$ & only $\presigmafg$ & 100 steps, $s=500$, lr $\rm = 10$, constant scheduler\\
    \bottomrule
    
    \end{tabular}}
    \label{tab:exp-imgnet-resnet50-hyperparam}
\end{table*}

\begin{table*}[t]
    \centering
    \caption{Training hyper-parameters of different MU methods w.r.t. results of instance unlearning in Table~\ref{tab:exp-timgnet-swint-instance} with Swin-T on Tiny-ImageNet.}
    \resizebox{0.85\textwidth}{!}{
    \begin{tabular}{c|cc|c}
    \toprule
    method & $\rmdata$ & $\fgdata$ & hyper-parameters\\
    \midrule
    $\pref$ & $\usym{1F5F8}$ & $\usym{1F5F8}$ & 20 epochs, lr $\rm ={10}^{-4}$, cosine scheduler \\
    $\exactf$ & $\usym{1F5F8}$ & $\usym{1F5F4}$ & 20 epochs, lr $\rm ={10}^{-4}$, cosine scheduler\\
    \midrule
    GA & $\usym{1F5F4}$ & $\usym{1F5F8}$ & 10 epochs, lr $\rm =6\times{10}^{-6}$ constant scheduler\\
    RL & $\usym{1F5F4}$ & $\usym{1F5F8}$ & 10 epochs, lr $\rm ={10}^{-5}$, cosine scheduler\\
    SalUn & $\usym{1F5F4}$ & $\usym{1F5F8}$ & 10 epochs, lr $\rm ={10}^{-5}$, cosine scheduler, saliency sparsity 50\%\\
    COUN+RL & $\usym{1F5F4}$ & $\usym{1F5F8}$ & 10 epochs, lr $\rm ={10}^{-5}$, cosine scheduler \\
    L2UL & $\usym{1F5F4}$ & $\usym{1F5F8}$ & lr $\rm ={10}^{-4}$, constant scheduler, regularization coefficient = 1.0\\
    DELETE & $\usym{1F5F4}$ & $\usym{1F5F8}$ & 10 epochs, lr $\rm =2\times{10}^{-5}$, cosine scheduler\\
    \midrule
    {\bf\ours} & only $\presigmarm$ & only $\presigmafg$ & 100 steps, $s=300$, lr $\rm = 1$, constant scheduler\\
    \bottomrule
    
    \end{tabular}}
    \label{tab:exp-timgnet-swint-instance-hyperparam}
\end{table*}

\subsection{Training Details of Instance Unlearning}
\label{sec:settings-instace}

For instance unlearning in Table~\ref{tab:exp-timgnet-swint-instance} with {Swin-T} on {Tiny-ImageNet}, the basic settings are with that in Section~\ref{sec:settings-class}.
1,000 samples from the total 100,000 training samples are randomly selected as the forgetting data $\fgdata$ (1\%).
Each of the repeated 3 experiments of Table~\ref{tab:exp-timgnet-swint-instance} is corresponding to different random seeds, i.e., 0, 1 and 2.
All compared methods use the AdamW optimizer \cite{loshchilovdecoupled} with a weight decay of 0.05 and a batch size of 128.
The hyper-parameters are listed in Table~\ref{tab:exp-timgnet-swint-instance-hyperparam}.
Our \ours uses the Riemannian Adam optimizer \cite{kochurov2020geoopt} with a weight decay of 0.05.

\subsection{Sensitivity Analysis into Weighted Losses for \ours}
\label{sec:appendix:weighted-loss}

The optimization objective of \eqref{eq:opt-obj} for \ours consists of two terms, $\objfg$ and $\objrm$, which target the minimization of forgetting data information and the maximization of remaining data information, respectively.
While the two terms are equally weighted by default, we explore a weighted sum for $\objfg$ and $\objrm$ controlled by a coefficient $\rho\in[0,1]$, defined as follows:
\begin{equation}
\begin{aligned}
\label{eq:opt-obj-weighted}
\min_{{\bf U}\in{\rm St}(d,s)}
{J}({\bf U})
=(1-\rho)\cdot
\underbrace{
\Bigl(
\frac
{{\rm Tr}({\bf U}^\top\presigmafg{\bf U})}
{{\rm Tr}(\presigmafg)}
\Bigr)^2}_{\objfg}
+\rho\cdot
\underbrace{
\Bigl(
\frac
{{\rm Tr}\left(\presigmarm-{\bf U}{\bf U}^\top\presigmarm{\bf U}{\bf U}^\top\right)}
{{\rm Tr}\left(\presigmarm\right)}
\Bigr)^2}_{\objrm},
\end{aligned}
\end{equation}

Figure \ref{fig:sensitivity-weight} illustrates the sensitivity of unlearning performance of \ours to the loss coefficient $\rho$.
In this experiment, 4 of 200 classes in Tiny-ImageNet is chosen as the forgetting data, following the settings of Table \ref{tab:exp-timgnet-swint-class}.
As observed, a lower $\rho$ emphasizes the elimination of information from the forgetting set ($\objfg$) with a performance degradation on the remaining set, while a higher $\rho$ prioritizes the preservation of remaining data ($\objrm$), hindering unlearning efficacy.
Additionally, our results suggest that assigning an insufficient weight to $\objfg$ (small values of $(1-\rho)$) can be detrimental (high accuracy on forgetting data), when dealing with a limited amount of forgetting data (2\%) in this case. 
We therefore suggest that a simple, equal-weighting scheme for $\objfg$ and $\objrm$ is typically sufficient and effective.

\begin{figure*}[t]
    \centering
    \includegraphics[width=0.95\linewidth]{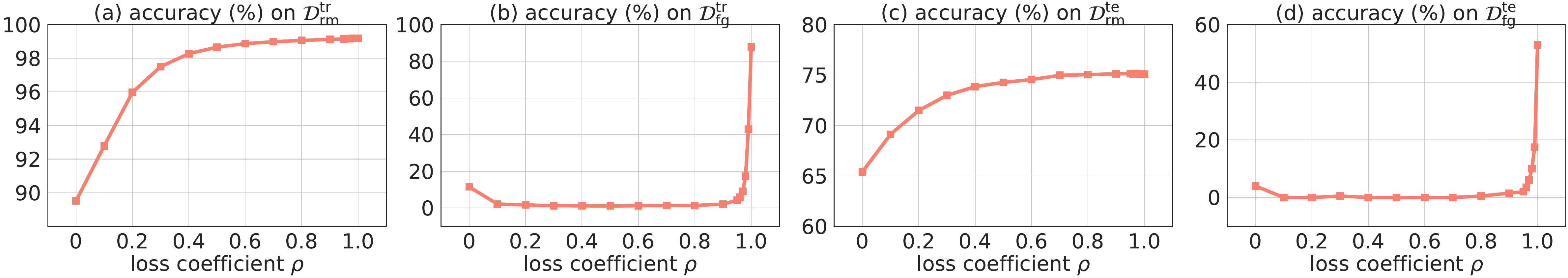}
    \caption{Unlearning performance w.r.t. varied values of the loss coefficient $\rho$.
    Empirical setups follow the same as that of Table \ref{tab:exp-timgnet-swint-class}.
    }
    \label{fig:sensitivity-weight}
\end{figure*}

\section{Extreme Unlearning and Continual Unlearning}
\label{app:sec:exp-extreme-continual}

\subsection{Extreme Unlearning}
\paragraph{Setups.}
We conduct evaluations to investigate the robustness of MU methods under an extreme scenario, where a very high ratio (over 90\%) of the training samples is required to be removed.
This situation may also arise in practice, such as deleting data across a large number of users on one occasion.
In this experiment, 180 of 200 classes (90\%) in Tiny-ImageNet are randomly selected as $\fgdata$, and experiments are repeated 3 times.
We adopt different random seeds, i.e., 0, 1 and 2, to guarantee that the forgetting 180 labels are different in each experiment.
All compared methods use the AdamW optimizer \cite{loshchilovdecoupled} with weight decay of 0.05 and a batch size of 128.
The hyper-parameters are listed in Table~\ref{tab:exp-timgnet-swint-extreme-class-hyperparam}.
In Table~\ref{tab:exp-timgnet-swint-extreme-class-hyperparam}, the $\fgdata$-based GA and SalUn involving both $\rmdata$ and $\fgdata$ are highly-sensitive to the choice of the learning rate under this extreme unlearning scenario.
To obtain satisfactory results, we use different learning rates for GA and SalUn in each experiment w.r.t. different random seeds.
For GA, the hyper-parameters are \{4 epochs, lr $\rm =3\times{10}^{-7}$\}, \{7 epochs, lr $\rm =2\times{10}^{-7}$\} and \{3 epochs, lr $\rm =4\times{10}^{-7}$\} w.r.t. random seeds 0, 1 and 2.
For SalUn based on $\rmdata$ and $\fgdata$, the hyper-parameters are \{lr $\rm =2\times{10}^{-4}$\}, \{lr $\rm =1.5\times{10}^{-4}$\} and \{lr $\rm ={10}^{-4}$\} w.r.t. random seeds 0, 1 and 2.
For L2UL, we only generate two adversarial examples per image for efficiency.
Our \ours optimizes $\mathbf{U}$  through the Riemannian Adam optimizer \cite{kochurov2020geoopt} with the weight decay of 0.05.

\paragraph{Results.}
Table~\ref{tab:exp-timgnet-swint-extreme-class} shows results of extreme unlearning with Swin-T on Tiny-ImageNet.
In this extreme unlearning scenario, methods that directly update $\pref$ struggle to maintain good performance on $\rmdata$, as the gradient updates to $\pref$ are dominated by penalizing performance on $\fgdata$ with a substantial number of samples. 
In contrast, our \ours does not modify $\pref$ and its projection matrix is optimized on the feature covariances w.r.t.  $\rmdata$ and $\fgdata$, which is less affected by the relative data sizes of $\rmdata$ and $\fgdata$.
\ours achieves the best average gap (Avg.G.) of 0.42 with about 0.05\% parameters w.r.t. the given model, and takes the shortest running time (RTE) about 0.29 seconds.

\begin{table*}[t]
    \centering
    \caption{Training hyper-parameters of different MU methods under extreme unlearning with Swin-T on Tiny-ImageNet.}
    \resizebox{0.95\textwidth}{!}{
    \begin{tabular}{c|cc|c}
    \toprule
    method & $\rmdata$ & $\fgdata$ & hyper-parameters\\
    \midrule
    $\pref$ & $\usym{1F5F8}$ & $\usym{1F5F8}$ & 20 epochs, lr $\rm ={10}^{-4}$, cosine scheduler \\
    $\exactf$ & $\usym{1F5F8}$ & $\usym{1F5F4}$ & 20 epochs, lr $\rm ={10}^{-4}$, cosine scheduler\\
    \midrule
    GA & $\usym{1F5F4}$ & $\usym{1F5F8}$ & constant scheduler\\
    GA & $\usym{1F5F8}$ & $\usym{1F5F8}$ & constant scheduler\\
    RL & $\usym{1F5F4}$ & $\usym{1F5F8}$ & 5 epochs, lr $\rm =5\times{10}^{-7}$, cosine scheduler\\
    RL & $\usym{1F5F8}$ & $\usym{1F5F8}$ & 10 epochs, lr $\rm =2\times{10}^{-5}$, cosine scheduler\\
    SalUn & $\usym{1F5F4}$ & $\usym{1F5F8}$ & 5 epochs, lr $\rm =5\times{10}^{-7}$, cosine scheduler, saliency sparsity 50\%\\
    SalUn & $\usym{1F5F8}$ & $\usym{1F5F8}$ & 10 epochs, cosine scheduler, saliency sparsity 50\%\\
    COUN+RL & $\usym{1F5F4}$ & $\usym{1F5F8}$ & 5 epochs, lr $\rm ={10}^{-4}$, cosine scheduler \\
    COUN+RL & $\usym{1F5F8}$ & $\usym{1F5F8}$ & 5 epochs, lr $\rm ={10}^{-4}$, cosine scheduler \\
    L2UL & $\usym{1F5F4}$ & $\usym{1F5F8}$ & lr $\rm =2\times{10}^{-4}$, constant scheduler, regularization coefficient = 1.0\\
    DELETE & $\usym{1F5F4}$ & $\usym{1F5F8}$ & 10 epochs, lr $\rm =2\times{10}^{-4}$, cosine scheduler\\
    \midrule
    {\bf\ours} & only $\presigmarm$ & only $\presigmafg$ & 50 steps, $s=18$, lr $\rm = 1$, constant scheduler\\
    \bottomrule
    
    \end{tabular}}
    \label{tab:exp-timgnet-swint-extreme-class-hyperparam}
\end{table*}

\begin{table*}[t]
    \centering
    \caption{Comparison on performance under extreme unlearning with {Swin-T} on {Tiny-ImageNet}.}
    \resizebox{0.99\textwidth}{!}{
    \begin{tabular}{c|cc|cc cc|c|ccc}
    \toprule
    method & $\rmdata$ & $\fgdata$ & $\rm\bf Acc_{rm}^{tr}$ & $\rm\bf Acc_{fg}^{tr}$ & $\rm\bf Acc_{rm}^{te}$ & $\rm\bf Acc_{fg}^{te}$ & {\bf MIA} & Avg.G.$\downarrow$ & RTE $\downarrow$ & \# param. (\%) \\
    \midrule
    $\pref$ & $\usym{1F5F8}$ & $\usym{1F5F8}$ & $99.64_{\pm0.02}$ & $99.63_{\pm0.00}$ & $75.43_{\pm2.61}$ & $74.45_{\pm0.29}$ & $95.34_{\pm0.43}$ & - &  882.05 & 27,673,154 (100\%)\\
    
    $\exactf$ & $\usym{1F5F8}$ & $\usym{1F5F4}$ & $99.97_{\pm0.01}$ (\textcolor{blue}{0.00}) & $0.00_{\pm0.00}$ (\textcolor{blue}{0.00}) & $90.57_{\pm0.75}$ (\textcolor{blue}{0.00}) & $0.00_{\pm0.00}$ (\textcolor{blue}{0.00}) & $0.00_{\pm0.00}$ (\textcolor{blue}{0.00}) & \textcolor{blue}{0.00} & 95.17 & 27,673,154 (100\%)\\
    \midrule
    GA & $\usym{1F5F4}$ & $\usym{1F5F8}$ & $14.26_{\pm6.06}$ (\textcolor{blue}{85.71}) & $9.61_{\pm4.36}$ (\textcolor{blue}{9.61}) & $12.20_{\pm4.61}$ (\textcolor{blue}{78.37}) & $8.10_{\pm3.46}$ (\textcolor{blue}{8.10}) & $11.41_{\pm5.19}$ (\textcolor{blue}{11.41}) & \textcolor{blue}{38.64} & 178.07 & 27,673,154 (100\%)\\

    \textcolor{gray}{GA} & \textcolor{gray}{$\usym{1F5F8}$} & \textcolor{gray}{$\usym{1F5F8}$} & \textcolor{gray}{$19.96_{\pm2.63}$ (\textcolor{grayblue}{80.01})} & \textcolor{gray}{$11.42_{\pm3.45}$ (\textcolor{grayblue}{11.42})} & \textcolor{gray}{$16.67_{\pm1.14}$ (\textcolor{grayblue}{73.90})} & \textcolor{gray}{$9.59_{\pm2.70}$ (\textcolor{grayblue}{9.59})} & \textcolor{gray}{$13.57_{\pm4.04}$ (\textcolor{grayblue}{13.57})} & \textcolor{grayblue}{37.70} & \textcolor{gray}{242.46} & \textcolor{gray}{27,673,154 (100\%)}\\


    RL & $\usym{1F5F4}$ & $\usym{1F5F8}$ & $26.98_{\pm3.11}$ (\textcolor{blue}{72.99}) & $16.33_{\pm2.69}$ (\textcolor{blue}{16.33}) & $23.43_{\pm2.05}$ (\textcolor{blue}{67.14}) & $13.75_{\pm2.35}$ (\textcolor{blue}{13.75}) & $5.39_{\pm0.75}$ (\textcolor{blue}{5.39}) & \textcolor{blue}{35.12} & 195.26 & 27,673,154 (100\%)\\

    \textcolor{gray}{RL} & \textcolor{gray}{$\usym{1F5F8}$} & \textcolor{gray}{$\usym{1F5F8}$} & \textcolor{gray}{$98.77_{\pm0.22}$ (\textcolor{grayblue}{1.20})} & \textcolor{gray}{$20.40_{\pm0.83}$ (\textcolor{grayblue}{20.40})} & \textcolor{gray}{$90.03_{\pm0.58}$ (\textcolor{grayblue}{0.54})} & \textcolor{gray}{$15.38_{\pm0.47}$ (\textcolor{grayblue}{15.38})}  & \textcolor{gray}{$1.10_{\pm0.51}$ (\textcolor{grayblue}{1.10})} & \textcolor{grayblue}{7.72} & \textcolor{gray}{431.45} & \textcolor{gray}{27,673,154} (\textcolor{gray}{100\%})\\

    SalUn & $\usym{1F5F4}$ & $\usym{1F5F8}$ & $27.39_{\pm0.48}$ (\textcolor{blue}{72.58}) & $17.57_{\pm3.28}$ (\textcolor{blue}{17.57}) & $23.30_{\pm2.1}$ (\textcolor{blue}{67.27}) & $15.05_{\pm2.65}$ (\textcolor{blue}{15.05}) & $12.25_{\pm1.57}$ (\textcolor{blue}{12.25}) & \textcolor{blue}{36.95} & 235.23 & 13,836,577 (50\%)\\

    \textcolor{gray}{SalUn} & \textcolor{gray}{$\usym{1F5F8}$} & \textcolor{gray}{$\usym{1F5F8}$} & \textcolor{gray}{$98.72_{\pm0.37}$ (\textcolor{grayblue}{1.25})} & \textcolor{gray}{$13.88_{\pm3.09}$ (\textcolor{grayblue}{13.88})} & \textcolor{gray}{$87.93_{\pm1.33}$ (\textcolor{grayblue}{2.64})} & \textcolor{gray}{$11.61_{\pm2.29}$ (\textcolor{grayblue}{11.61})} & \textcolor{gray}{$2.74_{\pm1.14}$ (\textcolor{grayblue}{2.14})} & \textcolor{grayblue}{6.31} & \textcolor{gray}{476.78} & \textcolor{gray}{13,836,577} (\textcolor{gray}{50\%})\\


    COUN+RL & $\usym{1F5F4}$ & $\usym{1F5F8}$ & $35.77_{\pm3.93}$ (\textcolor{blue}{64.20}) & $27.33_{\pm0.42}$ (\textcolor{blue}{27.33}) & $33.10_{\pm3.50}$ (\textcolor{blue}{57.47}) & $24.07_{\pm0.25}$ (\textcolor{blue}{24.07}) & $3.28_{\pm3.46}$ (\textcolor{blue}{3.28}) & \textcolor{blue}{35.27} & 337.69 & 27,673,154 (100\%)\\

    \textcolor{gray}{COUN+RL} & \textcolor{gray}{$\usym{1F5F8}$} & \textcolor{gray}{$\usym{1F5F8}$} & \textcolor{gray}{$99.86_{\pm0.05}$ (\textcolor{grayblue}{0.11})} & \textcolor{gray}{$24.50_{\pm0.63}$ (\textcolor{grayblue}{24.50})} & \textcolor{gray}{$88.50_{\pm1.61}$ (\textcolor{grayblue}{2.07})} & \textcolor{gray}{$20.01_{\pm0.89}$ (\textcolor{grayblue}{20.01})} & \textcolor{gray}{$4.05_{\pm0.44}$ (\textcolor{grayblue}{4.05})} & \textcolor{grayblue}{10.15} & \textcolor{gray}{357.45} & \textcolor{gray}{27,673,154 (100\%)}\\

    L2UL & $\usym{1F5F4}$ & $\usym{1F5F8}$ & $23.90_{\pm7.94}$ (\textcolor{blue}{76.07}) & $15.00_{\pm3.50}$ (\textcolor{blue}{15.00}) & $19.10_{\pm6.35}$ (\textcolor{blue}{71.47}) & $11.91_{\pm2.50}$ (\textcolor{blue}{11.91}) & $3.66_{\pm1.95}$ (\textcolor{blue}{3.66})  & \textcolor{blue}{35.63} & 7361.76 & 27,673,154 (100\%)\\ 
    
    DELETE & $\usym{1F5F4}$ & $\usym{1F5F8}$ & $97.37_{\pm0.36}$ (\textcolor{blue}{2.60}) & $6.95_{\pm0.31}$ (\textcolor{blue}{6.95}) & $81.07_{\pm1.53}$ (\textcolor{blue}{9.50}) & $10.96_{\pm0.48}$ (\textcolor{blue}{10.96}) & $1.16_{\pm0.16}$ (\textcolor{blue}{1.16}) & \textcolor{blue}{6.23} & 446.62 & 27,673,154 (100\%)\\
    \midrule
    {\bf\ours} & only $\presigmarm$ & only $\presigmafg$ & $99.63_{\pm0.18}$ (\textcolor{blue}{\bf0.34}) & $0.03_{\pm0.03}$ (\textcolor{blue}{\bf0.03}) & $92.17_{\pm0.65}$ (\textcolor{blue}{\bf1.60}) & $0.02_{\pm0.02}$ (\textcolor{blue}{\bf0.02}) & $0.08_{\pm0.12}$ (\textcolor{blue}{\bf0.08}) & \textcolor{blue}{\bf0.42} & {\bf0.29} & {\bf13,824} ({\bf0.05\%})\\
    \bottomrule
    \end{tabular}}
    \label{tab:exp-timgnet-swint-extreme-class}
\end{table*}

\begin{table*}[t]
    \centering
    \caption{Comparison results of Swin-T on Tiny-ImageNet under continual unlearning.}
    \resizebox{0.95\textwidth}{!}{
    \begin{tabular}{c|ccc|ccc|ccc}
    \toprule
    method & $\rmdata$ & $\fgdata$ & $\fgpdata$ & $\rm\bf Acc_{rm}^{tr}$ & $\rm\bf Acc_{fg}^{tr}$ & $\rm\bf Acc_{fgp}^{tr}$ & $\rm\bf Acc_{rm}^{te}$ & $\rm\bf Acc_{fg}^{te}$ & $\rm\bf Acc_{fgp}^{te}$ \\
    \midrule
    
    \multicolumn{10}{c}{\it Round-1}\\
    $\pref$ & $\usym{1F5F8}$ & $\usym{1F5F8}$ & - & 99.63 & 99.90 & - & 74.49 & 80.00 & - \\
    $\exactf$ & $\usym{1F5F8}$ & $\usym{1F5F4}$ & - & 99.65 (\textcolor{blue}{0.00}) & 0.00 (\textcolor{blue}{0.00}) & - & 75.08 (\textcolor{blue}{0.00}) & 0.00 (\textcolor{blue}{0.00}) & - \\

    RL & $\usym{1F5F4}$ & $\usym{1F5F8}$ & - & 97.00 (\textcolor{blue}{2.65}) & 8.50 (\textcolor{blue}{8.50}) & - & 70.87 (\textcolor{blue}{4.21}) & 1.00 (\textcolor{blue}{1.00}) & - \\

    \textcolor{gray}{RL} & \textcolor{gray}{$\usym{1F5F8}$} & \textcolor{gray}{$\usym{1F5F8}$} & - & \textcolor{gray}{99.93 (\textcolor{grayblue}{0.28})} & \textcolor{gray}{4.90 (\textcolor{grayblue}{4.90})} & \textcolor{gray}{-} & \textcolor{gray}{74.34 (\textcolor{grayblue}{0.74})} & \textcolor{gray}{2.00 (\textcolor{grayblue}{2.00})} & \textcolor{gray}{-} \\

     
    DELETE & $\usym{1F5F4}$ & $\usym{1F5F8}$ & - & 99.21 (\textcolor{blue}{0.44}) & 8.00 (\textcolor{blue}{8.00}) & - & 73.87 (\textcolor{blue}{0.62}) & 0.00 (\textcolor{blue}{0.00}) & - \\
    
    {\bf \ours} & $\presigmarm$ & $\presigmafg$ & - & 96.86 (\textcolor{blue}{2.77}) & 4.20 (\textcolor{blue}{4.20}) & - & 72.65 (\textcolor{blue}{2.43}) & 0.00 (\textcolor{blue}{0.00}) & - \\

    \midrule
    \multicolumn{10}{c}{\it Round-2}\\
    $\pref$ & $\usym{1F5F8}$ & $\usym{1F5F8}$ & $\usym{1F5F8}$ & 99.63 & 99.60 & 99.90 & 74.55 & 69.00 & 80.00 \\
    $\exactf$ & $\usym{1F5F8}$ & $\usym{1F5F4}$ & $\usym{1F5F4}$ & 99.69 (\textcolor{blue}{0.00}) & 0.00 (\textcolor{blue}{0.00}) & 0.00 (\textcolor{blue}{0.00}) & 75.57 (\textcolor{blue}{0.00}) & 0.00 (\textcolor{blue}{0.00}) & 0.00 (\textcolor{blue}{0.00}) \\

    RL & $\usym{1F5F4}$ & $\usym{1F5F8}$ & $\usym{1F5F4}$ & 80.30 (\textcolor{blue}{19.39}) & 6.30 (\textcolor{blue}{6.30}) & 1.00 (\textcolor{blue}{1.00}) & 59.68 (\textcolor{blue}{15.89}) & 2.00 (\textcolor{blue}{2.00}) & 0.00 (\textcolor{blue}{0.00}) \\

    \textcolor{gray}{RL} & \textcolor{gray}{$\usym{1F5F8}$} & \textcolor{gray}{$\usym{1F5F8}$} & \textcolor{gray}{$\usym{1F5F4}$} & \textcolor{gray}{99.98 (\textcolor{grayblue}{0.29})} & \textcolor{gray}{6.30 (\textcolor{grayblue}{6.30})} & \textcolor{gray}{1.00 (\textcolor{grayblue}{1.00})} & \textcolor{gray}{74.37 (\textcolor{grayblue}{1.20})} & \textcolor{gray}{1.00 (\textcolor{grayblue}{1.00})} & \textcolor{gray}{0.00 (\textcolor{grayblue}{0.00})} \\

     
    DELETE & $\usym{1F5F4}$ & $\usym{1F5F8}$ & $\usym{1F5F4}$ & 98.33 (\textcolor{blue}{1.30}) & 10.60 (\textcolor{blue}{10.60}) & 3.90 (\textcolor{blue}{3.90}) & 72.40 (\textcolor{blue}{3.17}) & 5.00 (\textcolor{blue}{5.00}) & 0.00 (\textcolor{blue}{0.00}) \\
    
    {\bf \ours} & $\presigmarm$ & $\presigmafg$ & $\presigmafgp$ & 95.77 (\textcolor{blue}{3.92}) & 7.50 (\textcolor{blue}{7.50}) & 0.50 (\textcolor{blue}{0.50}) & 72.01 (\textcolor{blue}{3.56}) & 4.00 (\textcolor{blue}{4.00}) & 1.00 (\textcolor{blue}{1.00}) \\

    \midrule
    \multicolumn{10}{c}{\it Round-3}\\
    $\pref$ & $\usym{1F5F8}$ & $\usym{1F5F8}$ & $\usym{1F5F8}$ & 99.62 & 99.60 & 99.75 & 74.53 & 69.00 & 74.50 \\
    $\exactf$ & $\usym{1F5F8}$ & $\usym{1F5F4}$ & $\usym{1F5F4}$ & 99.66 (\textcolor{blue}{0.00}) & 0.00 (\textcolor{blue}{0.00}) & 0.00 (\textcolor{blue}{0.00}) & 73.06 (\textcolor{blue}{0.00}) & 0.00 (\textcolor{blue}{0.00}) & 0.00 (\textcolor{blue}{0.00}) \\

    RL & $\usym{1F5F4}$ & $\usym{1F5F8}$ & $\usym{1F5F4}$ & 45.16 (\textcolor{blue}{54.50}) & 0.70 (\textcolor{blue}{0.70}) & 0.30 (\textcolor{blue}{0.30}) & 36.27 (\textcolor{blue}{36.79}) & 0.00 (\textcolor{blue}{0.00}) & 0.50 (\textcolor{blue}{0.50}) \\

    \textcolor{gray}{RL} & \textcolor{gray}{$\usym{1F5F8}$} & \textcolor{gray}{$\usym{1F5F8}$} & \textcolor{gray}{$\usym{1F5F4}$} & \textcolor{gray}{99.98 (\textcolor{grayblue}{0.32})} & \textcolor{gray}{7.70 (\textcolor{grayblue}{7.70})} & \textcolor{gray}{0.30 (\textcolor{grayblue}{0.30})} & \textcolor{gray}{73.79 (\textcolor{grayblue}{0.73})} & \textcolor{gray}{2.00 (\textcolor{grayblue}{2.00})} & \textcolor{gray}{0.00 (\textcolor{grayblue}{0.00})} \\

     
    DELETE & $\usym{1F5F4}$ & $\usym{1F5F8}$ & $\usym{1F5F4}$ & 96.27 (\textcolor{blue}{3.39}) & 5.10 (\textcolor{blue}{5.10}) & 3.85 (\textcolor{blue}{3.85}) & 70.05 (\textcolor{blue}{3.01}) & 3.00 (\textcolor{blue}{3.00}) & 1.50 (\textcolor{blue}{1.50}) \\
    
    {\bf \ours} & $\presigmarm$ & $\presigmafg$ & $\presigmafgp$ & 96.39 (\textcolor{blue}{3.27}) & 1.70 (\textcolor{blue}{1.70}) & 0.35 (\textcolor{blue}{0.35}) & 72.80 (\textcolor{blue}{0.26}) & 0.00 (\textcolor{blue}{0.00}) & 0.50 (\textcolor{blue}{0.50}) \\

    \midrule
    \multicolumn{10}{c}{\it Round-4}\\
    $\pref$ & $\usym{1F5F8}$ & $\usym{1F5F8}$ & $\usym{1F5F8}$ & 99.63 & 99.30 & 99.80 & 74.54 & 73.00 & 75.33 \\
    $\exactf$ & $\usym{1F5F8}$ & $\usym{1F5F4}$ & $\usym{1F5F4}$ & 99.67 (\textcolor{blue}{0.00}) & 0.00 (\textcolor{blue}{0.00}) & 0.00 (\textcolor{blue}{0.00}) & 75.72 (\textcolor{blue}{0.00}) & 0.00 (\textcolor{blue}{0.00}) & 0.00 (\textcolor{blue}{0.00}) \\

    RL & $\usym{1F5F4}$ & $\usym{1F5F8}$ & $\usym{1F5F4}$ & 10.96 (\textcolor{blue}{88.71}) & 0.30 (\textcolor{blue}{0.30}) & 0.00 (\textcolor{blue}{0.00}) & 9.97 (\textcolor{blue}{65.75}) & 1.00 (\textcolor{blue}{1.00}) & 0.00 (\textcolor{blue}{0.00}) \\

    \textcolor{gray}{RL} & \textcolor{gray}{$\usym{1F5F8}$} & \textcolor{gray}{$\usym{1F5F8}$} & \textcolor{gray}{$\usym{1F5F4}$} & \textcolor{gray}{99.98 (\textcolor{grayblue}{0.31})} & \textcolor{gray}{17.60 (\textcolor{grayblue}{17.60})} & \textcolor{gray}{1.07 (\textcolor{grayblue}{1.07})} & \textcolor{gray}{73.43 (\textcolor{grayblue}{2.29})} & \textcolor{gray}{9.00 (\textcolor{grayblue}{9.00})} & \textcolor{gray}{0.00 (\textcolor{grayblue}{0.00})} \\

     
    DELETE & $\usym{1F5F4}$ & $\usym{1F5F8}$ & $\usym{1F5F4}$ & 93.07 (\textcolor{blue}{6.60}) & 3.80 (\textcolor{blue}{3.80}) & 2.43 (\textcolor{blue}{2.43}) & 67.91 (\textcolor{blue}{7.81}) & 3.00 (\textcolor{blue}{3.00}) & 1.33 (\textcolor{blue}{1.33}) \\
    
    {\bf \ours} & $\presigmarm$ & $\presigmafg$ & $\presigmafgp$ & 96.48 (\textcolor{blue}{3.19}) & 5.50 (\textcolor{blue}{5.50}) & 0.07 (\textcolor{blue}{0.07}) & 72.50 (\textcolor{blue}{3.22}) & 3.00 (\textcolor{blue}{3.00}) & 0.00 (\textcolor{blue}{0.00}) \\
    \bottomrule
    \end{tabular}}
    \label{tab:exp-timgnet-swint-continual-class}
\end{table*}

\subsection{Continual Unlearning}

\paragraph{Setups.}
We provide evaluations for different MU methods under a practical scenario where multiple unlearning requests are in a continual order.
To be specific, we consider a 4-round continual unlearning.
In each round, 2 classes (1\%) are randomly selected as the forgetting data $\fgdata$, and the forgetting labels at each round are \{18,170\}, \{80,49\}, \{117,51\}, and \{16,118\}, respectively.
We introduce two new metrics $\rm Acc^{tr}_{fgp}$ and $\rm Acc^{te}_{fgp}$, i.e., accuracy on the forgetting data in previous rounds $\fgpdata$, to evaluate whether the model indeed forget data from early unlearning requests.
Results of continual unlearning are shown in Table~\ref{tab:exp-timgnet-swint-continual-class} including multiple mainstream MU methods.

In this continual unlearning, our \ours optimizes to guarantee the unlearning performance on forgetting data from previous rounds $\fgpdata$, where we introduce an additional loss $\objfgp$ to the objective of $J(\bf U)$ in (\ref{eq:opt-obj}) for achieving the task of continual unlearning as follows:
\begin{equation}
\label{eq:opt-obj-continual}
\min_{{\bf U}\in{\rm St}(d,s)}
\underbrace{
\Bigl(
\frac
{{\rm Tr}({\bf U}^\top\presigmafgp{\bf U})}
{{\rm Tr}(\presigmafgp)}
\Bigr)^2}_{\objfgp}
+
\underbrace{
\Bigl(
\frac
{{\rm Tr}({\bf U}^\top\presigmafg{\bf U})}
{{\rm Tr}(\presigmafg)}
\Bigr)^2}_{\objfg}
+
\underbrace{
\Bigl(
\frac
{{\rm Tr}\left(\presigmarm-{\bf U}{\bf U}^\top\presigmarm{\bf U}{\bf U}^\top\right)}
{{\rm Tr}\left(\presigmarm\right)}
\Bigr)^2}_{\objrm}.
\end{equation}
$\objfgp$ takes a similar form as $\objfg$ and measures the projected variance of the covariance matrix $\presigmafgp$ w.r.t. $\fgpdata$ captured within the subspace spanned by $\bf U$. Hence, minimizing $\objfgp$ achieves unlearning on $\fgpdata$.
Note that $\objfgp$ can be efficiently implemented without additional computation, since $\presigmafgp$ has been accessed in previous rounds.

\paragraph{Results.}
During the multiple rounds of unlearning, other MU methods constantly modify the parameters of the pretrained $\pref$ and directly affect the learning ability of $\pref$.
As a result, in Table~\ref{tab:exp-timgnet-swint-continual-class}, these methods show  significant accuracy drops on the remaining data $\rmdata$ and even slight accuracy increases on the forgetting data  $\fgpdata$ in previous rounds, after 4 rounds of unlearning.
In contrast, \ours does not update parameters of $\pref$, and well preserves the learning ability of $\pref$.
By only optimizing the projection matrix $\bf U$ given $\presigmarm$, $\presigmafg$ and $\presigmafgp$, the learned subspace of \ours successfully maintains the accuracy on $\rmdata$ and reduces the accuracy on $\fgdata$ and $\fgpdata$.

\section{Applications: Face Recognition and Emotion Recognition}
\label{app:sec:application}

Both identities and emotions are key facial attributes closely related to user privacy.
Exploring the unlearning of such information can benefit privacy protection.
Thereby we investigate MU for this two real-world applications of face recognition and emotion recognition.

\begin{table*}[t]
    \centering
    \caption{Training hyper-parameters of different MU methods w.r.t. Table~\ref{tab:exp-vggface2-resnet50} with ResNet50 on VGGFace2.}
    \resizebox{0.95\textwidth}{!}{
    \begin{tabular}{c|cc|c}
    \toprule
    method & $\rmdata$ & $\fgdata$ & hyper-parameters\\
    \midrule
    $\pref$ & $\usym{1F5F8}$ & $\usym{1F5F8}$ & 200 epochs, lr $\rm ={10}^{-1}$, cosine scheduler \\
    $\exactf$ & $\usym{1F5F8}$ & $\usym{1F5F4}$ & 200 epochs, lr $\rm ={10}^{-1}$, cosine scheduler\\
    \midrule
    GA & $\usym{1F5F4}$ & $\usym{1F5F8}$ & 5 epochs, lr $\rm =5\times{10}^{-4}$, constant scheduler\\
    GA & $\usym{1F5F8}$ & $\usym{1F5F8}$ & 5 epochs, lr $\rm =5\times{10}^{-4}$, constant scheduler\\
    RL & $\usym{1F5F4}$ & $\usym{1F5F8}$ & 10 epochs, lr $\rm =5\times{10}^{-4}$, cosine scheduler\\
    RL & $\usym{1F5F8}$ & $\usym{1F5F8}$ & 10 epochs, lr $\rm =5\times{10}^{-2}$, cosine scheduler\\
    SalUn & $\usym{1F5F4}$ & $\usym{1F5F8}$ & 10 epochs, lr $\rm ={10}^{-4}$, cosine scheduler, saliency sparsity 50\%\\
    SalUn & $\usym{1F5F8}$ & $\usym{1F5F8}$ & 10 epochs, lr $\rm =5\times{10}^{-2}$, cosine scheduler, saliency sparsity 50\%\\
    COUN+RL & $\usym{1F5F4}$ & $\usym{1F5F8}$ & 10 epochs, lr $\rm =2\times{10}^{-5}$, cosine scheduler\\
    COUN+RL & $\usym{1F5F8}$ & $\usym{1F5F8}$ & 15 epochs, lr $\rm =5\times{10}^{-2}$, cosine scheduler\\
    L2UL & $\usym{1F5F4}$ & $\usym{1F5F8}$ & lr $\rm =1\times{10}^{-4}$, constant scheduler, regularization coefficient = 1.0\\
    DELETE & $\usym{1F5F4}$ & $\usym{1F5F8}$ & 10 epochs, lr $\rm =5\times{10}^{-4}$, cosine scheduler\\
    \midrule
    {\bf\ours} & only $\presigmarm$ & only $\presigmafg$ & 50 steps, $s=21$, lr $\rm = 1$, constant scheduler\\
    \bottomrule
    
    \end{tabular}}
    \label{tab:exp-vggface2-resnet50-hyperparam}
\end{table*}

\begin{table*}[t]
    \centering
    \caption{Comparison results of ResNet50 on VGGFace2, forgetting one of 200 identities.}
    \resizebox{0.95\textwidth}{!}{
    \begin{tabular}{c|cc|cc cc|c|ccc}
    \toprule
    method & $\rmdata$ & $\fgdata$ & $\rm\bf Acc_{rm}^{tr}$ & $\rm\bf Acc_{fg}^{tr}$ & $\rm\bf Acc_{rm}^{te}$ & $\rm\bf Acc_{fg}^{te}$ & {\bf MIA} & Avg.G.$\downarrow$ & RTE $\downarrow$ & \# param. (\%) \\
    \midrule
    $\pref$ & $\usym{1F5F8}$ & $\usym{1F5F8}$ & 100.00 & 100.00 & 98.15 & 99.49 & 100.00 & - & 5970.95 & 25,557,032 (100\%)\\
    
    $\exactf$ & $\usym{1F5F8}$ & $\usym{1F5F4}$ & 100.00 (\textcolor{blue}{0.00}) & 0.00 (\textcolor{blue}{0.00}) & 98.21 (\textcolor{blue}{0.00}) & 0.00 (\textcolor{blue}{0.00}) & 0.00 (\textcolor{blue}{0.00}) & \textcolor{blue}{0.00} & 5870.35 & 25,557,032 (100\%)\\
    \midrule

    GA & $\usym{1F5F4}$ & $\usym{1F5F8}$ & 91.91 (\textcolor{blue}{8.09}) & 0.00 (\textcolor{blue}{0.00}) & 86.34 (\textcolor{blue}{11.87}) & 0.00 (\textcolor{blue}{0.00}) & 0.00 (\textcolor{blue}{0.00}) & \textcolor{blue}{4.00} & 3.35 & 25,557,032 (100\%)\\

    \textcolor{gray}{GA} & \textcolor{gray}{$\usym{1F5F8}$} & \textcolor{gray}{$\usym{1F5F8}$} & \textcolor{gray}{100.00 (\textcolor{grayblue}{0.00})} & \textcolor{gray}{0.25 (\textcolor{grayblue}{0.25})} & \textcolor{gray}{97.96 (\textcolor{grayblue}{0.25})} & \textcolor{gray}{2.03 (\textcolor{grayblue}{2.03})} & \textcolor{gray}{0.00 (\textcolor{grayblue}{0.00})} & \textcolor{grayblue}{0.51} & \textcolor{gray}{147.99} & \textcolor{gray}{25,557,032 (100\%)}\\

    
    RL & $\usym{1F5F4}$ & $\usym{1F5F8}$ & 99.64 (\textcolor{blue}{0.36}) & 11.25 (\textcolor{blue}{11.25}) & 95.27 (\textcolor{blue}{2.94}) & 13.20 (\textcolor{blue}{13.20}) & 0.00 (\textcolor{blue}{0.00}) & \textcolor{blue}{5.55} & 5.59 & 25,557,032 (100\%)\\

    \textcolor{gray}{RL} & \textcolor{gray}{$\usym{1F5F8}$} & \textcolor{gray}{$\usym{1F5F8}$} & \textcolor{gray}{99.11 (\textcolor{grayblue}{0.89})} & \textcolor{gray}{0.00 (\textcolor{grayblue}{0.00})} & \textcolor{gray}{94.66 (\textcolor{grayblue}{3.55})} & \textcolor{gray}{0.00 (\textcolor{grayblue}{0.00})} & \textcolor{gray}{0.00 (\textcolor{grayblue}{0.00})} & \textcolor{grayblue}{0.89} & \textcolor{gray}{303.81}& \textcolor{gray}{25,557,032} (\textcolor{gray}{100\%})\\

    SalUn & $\usym{1F5F4}$ & $\usym{1F5F8}$ & 99.68 (\textcolor{blue}{\bf0.32}) & 5.50 (\textcolor{blue}{5.50}) & 95.59 (\textcolor{blue}{\bf2.62}) & 5.08 (\textcolor{blue}{5.08}) & 0.00 (\textcolor{blue}{0.00}) & \textcolor{blue}{2.70} & 9.99& 12,778,516 (50\%)\\

    \textcolor{gray}{SalUn} & \textcolor{gray}{$\usym{1F5F8}$} & \textcolor{gray}{$\usym{1F5F8}$} & \textcolor{gray}{99.96 (\textcolor{grayblue}{0.04})} & \textcolor{gray}{0.00 (\textcolor{grayblue}{0.00})} & \textcolor{gray}{97.00 (\textcolor{grayblue}{1.21})} & \textcolor{gray}{0.00 (\textcolor{grayblue}{0.00})} & \textcolor{gray}{0.00 (\textcolor{grayblue}{0.00})} & \textcolor{grayblue}{0.25} & \textcolor{gray}{312.71}& \textcolor{gray}{12,778,516} (\textcolor{gray}{50\%})\\


    COUN+RL & $\usym{1F5F4}$ & $\usym{1F5F8}$ & 99.36 (\textcolor{blue}{0.64}) & 2.75 (\textcolor{blue}{2.75}) & 94.80 (\textcolor{blue}{3.41}) & 1.52 (\textcolor{blue}{1.52}) & 0.00 (\textcolor{blue}{0.00}) & \textcolor{blue}{1.66} & 8.11 & 25,557,032 (100\%)\\

    \textcolor{gray}{COUN+RL} & \textcolor{gray}{$\usym{1F5F8}$} & \textcolor{gray}{$\usym{1F5F8}$} & \textcolor{gray}{99.35 (\textcolor{grayblue}{0.65})} & \textcolor{gray}{0.00 (\textcolor{grayblue}{0.00})} & \textcolor{gray}{91.34 (\textcolor{grayblue}{6.87})} & \textcolor{gray}{0.00 (\textcolor{grayblue}{0.00})} & \textcolor{gray}{0.00 (\textcolor{grayblue}{0.00})} & \textcolor{grayblue}{1.50} & \textcolor{gray}{956.15} & \textcolor{gray}{25,557,032 (100\%)}\\

    L2UL & $\usym{1F5F4}$ & $\usym{1F5F8}$ & 98.09 (\textcolor{blue}{1.91}) & 2.75 (\textcolor{blue}{2.75}) & 92.88 (\textcolor{blue}{5.33}) & 3.55 (\textcolor{blue}{3.55}) & 0.00 (\textcolor{blue}{0.00}) & \textcolor{blue}{2.71} & 176.69& 25,557,032 (100\%)\\
    
    DELETE & $\usym{1F5F4}$ & $\usym{1F5F8}$ & 99.50 (\textcolor{blue}{0.50}) & 5.50 (\textcolor{blue}{5.50}) & 95.15 (\textcolor{blue}{3.06}) & 5.58 (\textcolor{blue}{5.58}) & 0.00 (\textcolor{blue}{0.00}) & \textcolor{blue}{2.93} & 5.82& 25,557,032 (100\%)\\
    \midrule
    {\bf \ours} & only $\presigmarm$ & only $\presigmafg$ & 97.20 (\textcolor{blue}{2.80}) & 0.00 (\textcolor{blue}{\bf0.00}) & 94.22 (\textcolor{blue}{3.99}) & 0.51 (\textcolor{blue}{\bf0.51}) & 0.00 (\textcolor{blue}{\bf0.00}) & \textcolor{blue}{\bf1.46} & {\bf0.45} &{\bf43,008} ({\bf0.17\%})\\
    \bottomrule
    \end{tabular}}
    \label{tab:exp-vggface2-resnet50}
\end{table*}

\subsection{Face Recognition with Machine Unlearning}

\paragraph{Setups.}
Experiments are conducted on a prevalent face identity recognition dataset VGGFace2 \cite{cao2018vggface2}, including approximately 3.31 million images collected from 9,131 identities.
Particularly, we adopt its publicly-available version\footnote{https://www.kaggle.com/datasets/yakhyokhuja/vggface2-112x112}, where all face images are aligned and cropped to $112\times112\times3$.
Our experiments utilize a filtered set of 200 identities, each with over 500 images.
Then, we randomly select 400 images for each identity, forming a training set containing 80,000 images, and the rest is adopted as the test set.
One of the 200 identities (0.5\%) is randomly selected as the forgetting data $\fgdata$, which is a challenging setting for class-centric unlearning.
The adopted neural network model is ResNet50 \cite{he2016deep}.
All the compared methods use the SGD optimizer \cite{bottou2012stochastic} with a weight decay of $5\times10^{-4}$, the momentum of 0.9, and a batch size of 128.
More training details are in Table~\ref{tab:exp-vggface2-resnet50-hyperparam} with comparison results in Table~\ref{tab:exp-vggface2-resnet50}.

\paragraph{Results.}
This unlearning setting, i.e., forgetting one of the 200 identities, implies a substantial size of $\rmdata$.
Therefore, MU methods  involving $\rmdata$ can easily achieve the ideal performance of the exact unlearning model of $\exactf$, as illustrated by GA (w/ $\rmdata$), RL (w/ $\rmdata$), SalUn (w/ $\rmdata$), and COUN+RL (w/ $\rmdata$) in Table~\ref{tab:exp-vggface2-resnet50}.
In contrast, it is challenging for MU methods by only accessing $\fgdata$ under this unlearning setting, as demonstrated by the relatively high accuracy on $\fgdata$ of GA (w/o $\rmdata$), RL (w/o $\rmdata$), SalUn (w/o $\rmdata$), COUN+RL (w/o $\rmdata$), L2UL and DELETE in Table~\ref{tab:exp-vggface2-resnet50}.
Our \ours does not visit the original samples and learns to preserve and remove essential patterns from the feature covariance matrices of $\rmdata$ and $\fgdata$ through low-dimensional projections, managing to simultaneously maintain high accuracy on $\rmdata$ and suppress accuracy on $\fgdata$.

\subsection{Emotion Recognition with Machine Unlearning}

\paragraph{Setups.}
Experiments are conducted on a popular emotion recognition dataset RAF-DB \cite{li2017reliable},
RAF-DB contains 12,271 and 3,068 training and test images with a size of $100\times100\times3$, respectively, within 7 different emotions: ``surprise'', ``fear'', ``disgust'', ``happiness'', ``sadness'', ``anger'' and ``neutral''.
In the experiments, one emotion is randomly selected as the forgetting data $\fgdata$ (``fear'').
The adopted model is ResNet18 \cite{he2016deep}.
All the compared methods use the SGD optimizer \cite{bottou2012stochastic} with a weight decay of $5\times10^{-4}$, the momentum of 0.9, and a batch size of 128.
More training details are listed in  Table~\ref{tab:exp-rafdb-resnet18-hyperparam}.
The comparison results among different MU methods are shown in Table~\ref{tab:exp-rafdb-resnet18}.

\paragraph{Results.}
As illustrated in Table~\ref{tab:exp-rafdb-resnet18}, when forgetting the ``fear'' emotion from the total 7 emotions, \ours shows the best approximation performance towards that of the retrained model $\exactf$.
Other MU methods either  heavily rely on the remaining data $\rmdata$ (RL and SalUn), or show unsatisfactory results on $\fgdata$ (L2UL and DELETE).
\ours explores the low-dimensional subspace in which the  features of $\fgdata$ and $\rmdata$ can be well distinguished through its reconstruction performances, achieving distinctively better unlearning performances.

\begin{table*}[t]
    \centering
    \caption{Training hyper-parameters of different MU methods w.r.t. Table~\ref{tab:exp-rafdb-resnet18} with ResNet18 on RAF-DB.}
    \resizebox{0.95\textwidth}{!}{
    \begin{tabular}{c|cc|c}
    \toprule
    method & $\rmdata$ & $\fgdata$ & hyper-parameters\\
    \midrule
    $\pref$ & $\usym{1F5F8}$ & $\usym{1F5F8}$ & 200 epochs, lr $\rm ={10}^{-1}$, cosine scheduler \\
    $\exactf$ & $\usym{1F5F8}$ & $\usym{1F5F4}$ & 200 epochs, lr $\rm ={10}^{-1}$, cosine scheduler\\
    \midrule
    GA & $\usym{1F5F4}$ & $\usym{1F5F8}$ & 10 epochs, lr $\rm =5\times{10}^{-4}$, constant scheduler\\
    GA & $\usym{1F5F8}$ & $\usym{1F5F8}$ & 10 epochs, lr $\rm =5\times{10}^{-3}$, constant scheduler\\
    RL & $\usym{1F5F4}$ & $\usym{1F5F8}$ & 20 epochs, lr $\rm =5\times{10}^{-4}$, cosine scheduler\\
    RL & $\usym{1F5F8}$ & $\usym{1F5F8}$ & 10 epochs, lr $\rm =5\times{10}^{-2}$, cosine scheduler\\
    SalUn & $\usym{1F5F4}$ & $\usym{1F5F8}$ & 20 epochs, lr $\rm ={10}^{-4}$, cosine scheduler, saliency sparsity 50\%\\
    SalUn & $\usym{1F5F8}$ & $\usym{1F5F8}$ & 10 epochs, lr $\rm =5\times{10}^{-2}$, cosine scheduler, saliency sparsity 50\%\\

    COUN+RL & $\usym{1F5F4}$ & $\usym{1F5F8}$ & 10 epochs, lr $\rm =2\times{10}^{-5}$, cosine scheduler\\
    
    COUN+RL & $\usym{1F5F8}$ & $\usym{1F5F8}$ & 15 epochs, lr $\rm =5\times{10}^{-2}$, cosine scheduler\\
    
    L2UL & $\usym{1F5F4}$ & $\usym{1F5F8}$ & lr $\rm =2\times{10}^{-4}$, constant scheduler, regularization coefficient = 1.0\\
    DELETE & $\usym{1F5F4}$ & $\usym{1F5F8}$ & 20 epochs, lr $\rm =2\times{10}^{-4}$, cosine scheduler\\
    \midrule
    {\bf\ours} & only $\presigmarm$ & only $\presigmafg$ & 50 steps, $s=6$, lr $\rm = 1$, constant scheduler\\
    \bottomrule
    
    \end{tabular}}
    \label{tab:exp-rafdb-resnet18-hyperparam}
\end{table*}

\begin{table*}[t]
    \centering
    \caption{Comparison results of ResNet18 on RAF-DB.}
    \resizebox{0.95\textwidth}{!}{
    \begin{tabular}{c|cc|cc cc | c|c}
    \toprule
    method & $\rmdata$ & $\fgdata$ & $\rm\bf Acc_{rm}^{tr}$ & $\rm\bf Acc_{fg}^{tr}$ & $\rm\bf Acc_{rm}^{te}$ & $\rm\bf Acc_{fg}^{te}$ & {\bf MIA} & Avg.G.\\
    \midrule
    $\pref$ & $\usym{1F5F8}$ & $\usym{1F5F8}$ & 100.00 & 100.00 & 85.10 & 56.76 & 100.00 & - \\
    
    $\exactf$ & $\usym{1F5F8}$ & $\usym{1F5F4}$ & 100.00 (\textcolor{blue}{0.00}) & 0.00 (\textcolor{blue}{0.00}) & 84.90 (\textcolor{blue}{0.00}) & 0.00 (\textcolor{blue}{0.00}) & 0.00 (\textcolor{blue}{0.00}) & \textcolor{blue}{0.00}\\
    \midrule
    
    GA & $\usym{1F5F4}$ & $\usym{1F5F8}$ & 91.38 (\textcolor{blue}{8.62}) & 0.36 (\textcolor{blue}{0.36}) & 76.75 (\textcolor{blue}{8.15}) & 0.00 (\textcolor{blue}{0.00}) & 0.36 (\textcolor{blue}{0.36}) & \textcolor{blue}{3.50}\\

    \textcolor{gray}{GA} & \textcolor{gray}{$\usym{1F5F8}$} & \textcolor{gray}{$\usym{1F5F8}$} & \textcolor{gray}{98.39 (\textcolor{grayblue}{1.61})} & \textcolor{gray}{0.00 (\textcolor{grayblue}{0.00})} & \textcolor{gray}{82.00 (\textcolor{grayblue}{2.90})} & \textcolor{gray}{0.00 (\textcolor{grayblue}{0.00})} & \textcolor{gray}{0.00 (\textcolor{grayblue}{0.00})} & \textcolor{grayblue}{0.90}\\


    RL & $\usym{1F5F4}$ & $\usym{1F5F8}$ & 98.96 (\textcolor{blue}{1.04}) & 15.66 (\textcolor{blue}{15.66}) & 79.93 (\textcolor{blue}{4.97}) & 5.41 (\textcolor{blue}{5.41}) & 0.00 (\textcolor{blue}{0.00}) & \textcolor{blue}{5.42}\\

    \textcolor{gray}{RL} & \textcolor{gray}{$\usym{1F5F8}$} & \textcolor{gray}{$\usym{1F5F8}$} & \textcolor{gray}{99.92 (\textcolor{grayblue}{0.08})} & \textcolor{gray}{0.00 (\textcolor{grayblue}{0.00})} & \textcolor{gray}{84.10 (\textcolor{grayblue}{0.80})} & \textcolor{gray}{0.00 (\textcolor{grayblue}{0.00})} & \textcolor{gray}{0.00 (\textcolor{grayblue}{0.00})} & \textcolor{grayblue}{0.18}\\

    SalUn & $\usym{1F5F4}$ & $\usym{1F5F8}$ & 99.83 (\textcolor{blue}{0.17}) & 12.46 (\textcolor{blue}{12.46}) & 82.53 (\textcolor{blue}{2.37}) & 2.70 (\textcolor{blue}{2.70}) & 0.36 (\textcolor{blue}{0.36}) & \textcolor{blue}{3.61}\\

    \textcolor{gray}{SalUn} & \textcolor{gray}{$\usym{1F5F8}$} & \textcolor{gray}{$\usym{1F5F8}$} & \textcolor{gray}{99.98 (\textcolor{grayblue}{0.02})} & \textcolor{gray}{0.00 (\textcolor{grayblue}{0.00})} & \textcolor{gray}{84.57 (\textcolor{grayblue}{0.33})} & \textcolor{gray}{0.00 (\textcolor{grayblue}{0.00})} & \textcolor{gray}{0.00 (\textcolor{grayblue}{0.00})} & \textcolor{grayblue}{0.07} \\


    COUN+RL & $\usym{1F5F4}$ & $\usym{1F5F8}$ & 99.57 (\textcolor{blue}{0.43}) & 6.76 (\textcolor{blue}{6.76}) & 81.63 (\textcolor{blue}{3.27}) & 1.35 (\textcolor{blue}{1.35}) & 1.07 (\textcolor{blue}{1.07}) & \textcolor{blue}{2.58}\\

    \textcolor{gray}{COUN+RL} & \textcolor{gray}{$\usym{1F5F8}$} & \textcolor{gray}{$\usym{1F5F8}$} & \textcolor{gray}{92.04 (\textcolor{grayblue}{7.96})} & \textcolor{gray}{0.00 (\textcolor{grayblue}{0.00})} & \textcolor{gray}{83.20 (\textcolor{grayblue}{1.70})} & \textcolor{gray}{0.00 (\textcolor{grayblue}{0.00})} & \textcolor{gray}{0.00 (\textcolor{grayblue}{0.00})} & \textcolor{grayblue}{1.93}\\

    L2UL & $\usym{1F5F4}$ & $\usym{1F5F8}$ & 99.99 (\textcolor{blue}{0.01}) & 10.68 (\textcolor{blue}{10.68}) & 84.03 (\textcolor{blue}{0.87}) & 1.35 (\textcolor{blue}{1.35}) & 0.00 (\textcolor{blue}{0.00}) & \textcolor{blue}{2.58}\\
    
    DELETE & $\usym{1F5F4}$ & $\usym{1F5F8}$ & 99.72 (\textcolor{blue}{0.28}) & 9.25 (\textcolor{blue}{9.25}) & 82.20 (\textcolor{blue}{2.70}) & 2.70 (\textcolor{blue}{2.70}) & 1.42 (\textcolor{blue}{1.42}) & \textcolor{blue}{3.27}\\
    \midrule
    {\bf \ours} & only $\presigmarm$ & only $\presigmafg$ & 100.00 (\textcolor{blue}{\bf0.00}) & 0.00 (\textcolor{blue}{\bf0.00}) & 84.90 (\textcolor{blue}{\bf0.00}) & 0.00 (\textcolor{blue}{\bf0.00}) & 0.00 (\textcolor{blue}{\bf0.00}) & \textcolor{blue}{\bf0.00}\\
    \bottomrule
    \end{tabular}}
    \label{tab:exp-rafdb-resnet18}
\end{table*}

\section{Discussions: Privacy Scope, Threat Model and Limitations}
\label{app:sec:privacy-threat}

\ours is {\it not} a certified unlearning method and does not provide formal privacy guarantees such as differential privacy or certified deletion.
The privacy-related advantage of \ours should instead be understood in a narrower sense: 
compared with mainstream approximate unlearning methods that repeatedly visit raw forgetting and remaining samples and update the full pretrained model through multiple forward-backward passes, \ours requires only one-shot feature extraction to form covariance matrices and subsequently optimizes a lightweight projection matrix without repeatedly revisiting raw data.
Therefore \ours primarily reduces the exposure surface of the unlearning pipeline and mitigates privacy risk associated with repeated raw-data access during unlearning.

\paragraph{Black-box adversaries.}
In common practical deployment settings, the attacker may only access model outputs through queries.
Such a black-box threat model is not aware of the inserted projection matrix $\bf U$ in \ours.
Under this black-box setting, we empirically audit privacy leakage using Membership Inference Attacks (MIAs).
Across our experiments in Section \ref{sec:exp}, \ours attains a small gap to the retrained model under a classic MIA method \cite{shokri2017membership}, indicating that the learned feature projection can suppress information about the forgetting data in a way that is difficult to exploit through output-based attacks.
In the following Section \ref{app:sec:mia-lira-rmia}, we provide additional evaluations against two advanced MIA methods, LiRA \cite{carlini2022membership} and RMIA \cite{zarifzadeh2024low}, demonstrating that our \ours maintains its superior privacy protection capability.


\paragraph{White-box adversaries and deployment considerations.}
A stronger threat model considers white-box adversaries with access to the internal components of the unlearned model, including its architecture, parameters, and even gradients. 
For \ours, we distinguish between different levels of exposure. 
Access to the projection matrix $\mathbf U$ alone does not immediately enable recovery of the forgotten information, since $\mathbf U$ specifies only a feature projection and does not directly contain the raw forgetting data or the pretrained model parameters. 
Exploiting $\mathbf U$ would generally require additional knowledge of the model architecture and the pretrained feature representations. 
If both $\mathbf U$ and the complete pretrained model are exposed, an adversary may simply remove the plug-in projection to recover the original model behavior. 
To mitigate this risk, the projector $\mathbf U\mathbf U^\top$ can be absorbed into the subsequent linear classifier before deployment, so that $\mathbf U$ is no longer released as an independent and removable module. 
This reduces the practical attack surface, although it does not provide a formal privacy guarantee against white-box adversaries. 
Accordingly, our empirical privacy analysis focuses on the commonly considered black-box MIAs, while rigorous evaluation and defense under stronger white-box threat models remain important directions for future work.

\subsection*{Additional Results against More MIAs}
\label{app:sec:mia-lira-rmia}

We further evaluate unlearned models using two stronger score-based black-box MIAs, LiRA \cite{carlini2022membership} and RMIA \cite{zarifzadeh2024low}, to assess whether the forgetting samples remain distinguishable under black-box threat models.
LiRA \cite{carlini2022membership} estimates the likelihood of membership from a model-centric perspective by training multiple ``shadow-in'' and ``shadow-out'' models and comparing the target model's confidence on a query sample against the score distributions of member and non-member cases induced by these shadow models.
RMIA \cite{zarifzadeh2024low}, in contrast, takes a data-side perspective and calibrates the membership signal using a reference population dataset, and infers membership by comparing the target model's response on a candidate sample with its responses on population samples.

\paragraph{Setups.}
Our setup follows Table~\ref{tab:exp-timgnet-swint-class}, where 4 out of 200 classes in Tiny-ImageNet are randomly selected as $\fgdata$.
For LiRA, we train 16 ``shadow-in'' models on the full training set and 16 ``shadow-out'' models on the remaining training set, each with different random seeds.
For RMIA, we train a reference model on the remaining training set, and utilize the remaining test samples as the population set.
We evaluate LiRA and RMIA by distinguishing the training forgetting samples from the test forgetting samples.
Figure~\ref{fig:mia} reports the ROC curves and AUC values of different unlearning methods under LiRA and RMIA.
For fairness, GA, RL, SalUn, and COUN+RL are all optimized with access to the training remaining data.

\paragraph{Results.}
As shown in Figure~\ref{fig:mia}, \ours achieves AUC values close to 0.5 under both LiRA and RMIA, indicating that these stronger black-box MIAs can hardly distinguish the forgetting samples from non-members after unlearning.
This result suggests that \ours effectively suppresses exploitable membership signals on the forgetting data.
By exploiting the separability between the remaining and forgetting data in the feature subspace, \ours removes feature components that are particularly informative for the forgetting set, thereby making membership inference substantially more difficult in a black-box setting.

Additionally, under LiRA, the retrained model achieves near‑zero AUROC, potentially because retraining completely removes the forget‑set classes, leaving the model unable to discriminate those samples. Consequently, its predictive confidence for forgotten members is uniformly and substantially lower than for non‑members.

\begin{figure*}[t]
    \centering
    \includegraphics[width=0.8\linewidth]{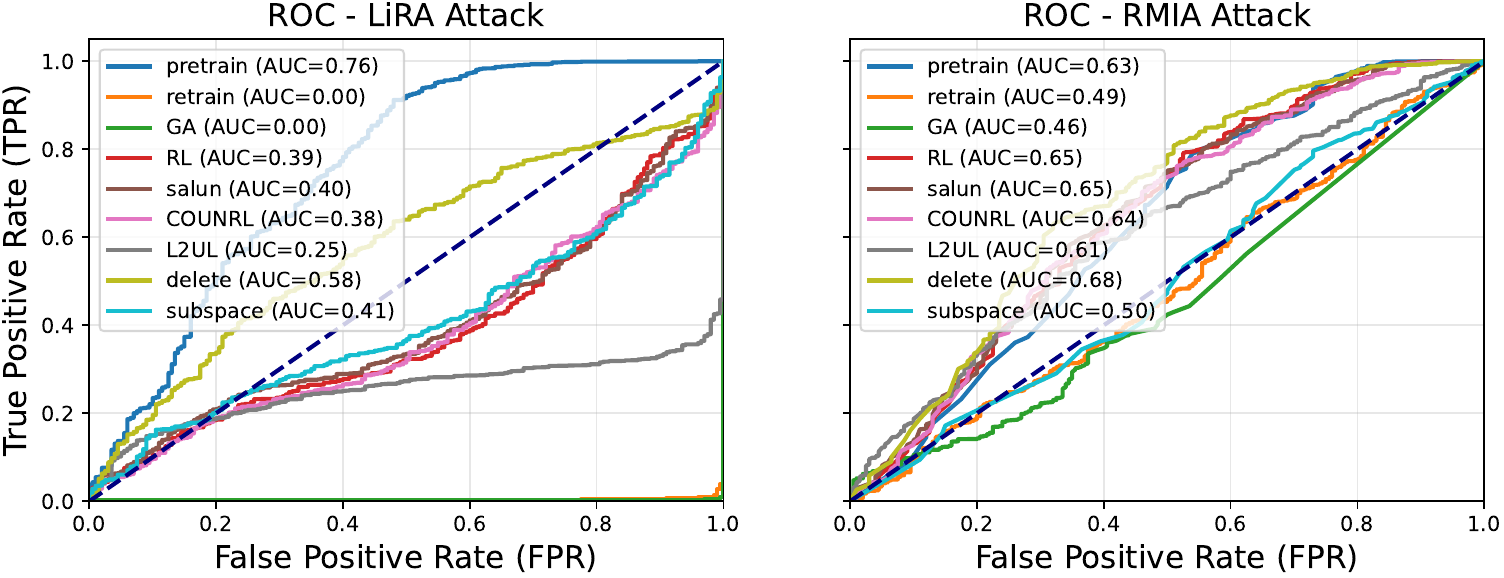}
    \caption{Results against two advanced MIAs, LiRA and RMIA.
    Empirical setups follow the same as that of Table \ref{tab:exp-timgnet-swint-class}.
    }
    \label{fig:mia}
\end{figure*}

\section{Low-Dimensional Feature Subspaces for MU in Generation Tasks}
\label{app:sec:generation}

In this section, we provide a preliminary exploration of using low-dimensional feature subspaces for machine unlearning in generation tasks.
As our \ours does not update the parameters of the pretrained model $\pref$, \ours is not directly applicable to models in generation tasks. 
Meanwhile, there are distinctively different generation models in different tasks.
For example, in image generation, the convolution network in DDPM \cite{ho2020denoising} and LDM \cite{rombach2022high} significantly differs from the vision transformer in DiT \cite{peebles2023scalable}, while the transformer itself functions differently between image generation and language-related generation tasks \cite{liu2025rethinking}.
Another challenge is that generative models are particularly sensitive to internal changes in their feature representations. For instance, even subtle alterations in the intermediate layers of a DDPM can lead to artifacts or meaningless pixels in the generated images.
Therefore, crafting an appropriate low-dimensional feature subspace must be carefully tailored to the specific generative model and the particular unlearning scenario.
This direction, to the best of our knowledge, still remains underexplored.
In the following, we present promising unlearning results from \ours applied to a {\it machine translation} task.
This initial success suggests the potential of our feature subspace perspective for MU in generative tasks, and we hope it will advocate further research in this area.

\paragraph{Setups.}
We evaluate \ours on a machine translation task with an unlearning target as {\it forcing the model to forget the ability of translating a specific language}. 
This natural language processing task is not a classification one and the subspace shall not be tightly coupled with class-specific features. 
The experiment is conducted on the WMT-19 dataset\footnote{https://huggingface.co/datasets/wmt/wmt19} and a modern large model named {\tt m2m100\_1.2B}\footnote{https://huggingface.co/facebook/m2m100\_1.2B}. 
The pretrained {\tt m2m100\_1.2B} model is fine-tuned in a parameter-efficient way (i.e., LoRA) on two language pairs from WMT-19: translating Czech into English and translating Lithuanian into English. 
This fine-tuned {\tt m2m100\_1.2B} model is adopted as the initial startpoint to be unlearned, and the to-be-unlearned language is Lithuanian.

\paragraph{Implementation.}
With this modern large model of {\tt m2m100\_1.2B}, we apply \ours to the {\it hidden states} from the output of the encoder of {\tt m2m100\_1.2B}. 
Specifically, the token-level hidden states get weighted averaged by the self-attention from the encoder and lead to weighted averaged 1,024-dimensional features, where 1,024 is the hidden dimension of {\tt m2m100\_1.2B}. 
In this way, we avoid the massive tokens and extremely large vocabulary size and focus on the attention-based weighted average features to apply \ours. 
Then, two feature covariance matrices are calculated w.r.t. Czech and Lithuanian inputs, and the optimization objective of \ours in \eqref{eq:opt-obj} is conducted to learn a projection matrix. 
In inference, the projector is employed at the token-level hidden states to proceed the propagation, similar as the way in Section \ref{sec:implement}.

\paragraph{Results.}
Our results are presented in the following Table \ref{tab:machine-translation}. 
We evaluate translation quality using the BLEU score \cite{papineni2002bleu} on the WMT-19 validation set for Czech-to-English (cs$\rightarrow$en) and Lithuanian-to-English (lt$\rightarrow$en) directions. 
The forgetting data (fg) and remaining (rm) data correspond to the lt$\rightarrow$en and cs→en pairs, respectively. 
Here, fully-trained denotes {\tt m2m100\_1.2B} fine-tuned on both language pairs, while retrained refers to {\tt m2m100\_1.2B} trained exclusively on the remaining (cs$\rightarrow$en) data without seeing (lt$\rightarrow$en) inputs. 
Blue numbers in parentheses denote the BLEU score gap relative to the retrained method (lower is better). 
After unlearning, our \ours achieves performance on the forgetting language pair (lt$\rightarrow$en) that is very close to the retrained model, implying successful forgetting of the Lithuanian-to-English translation capability. 
Meanwhile, \ours maintains good performance on the remaining Czech-to-English direction, demonstrating effective knowledge preservation. 
As an exemplary trial on large language models, this experiment validates \ours's potential for extension to non-classification tasks and its scalability to modern large models.

\begin{table*}[h]
\centering
\caption{Results of \ours on machine translation with a modern large model of {\tt m2m100\_1.2B}.}
\begin{tabular}{c|cc}
\toprule
method & cs$\rightarrow$en (rm) & lt$\rightarrow$en (fg) \\
\midrule
fully-trained & 27.63 & 17.54 \\
retrained & 26.19 & 12.31 \\
GA & 23.51 (\textcolor{blue}{2.68}) & 15.79 (\textcolor{blue}{3.48}) \\
\ours & 22.88 (\textcolor{blue}{3.31}) & 13.04 (\textcolor{blue}{0.73}) \\
\bottomrule
\end{tabular}
\label{tab:machine-translation}
\end{table*}

\section*{Impact Statement}

This paper presents work aimed at advancing the privacy protection of modern deep learning models by integrating classic machine learning techniques. 
This advancement has the potential to benefit a wide range of fields and societal applications that is related to user privacy protection. 
By releasing our code along with this work, we aim to provide the machine unlearning research community with a fresh methodological perspective and an interpretable and practical plug-in module that can be optimized with superior effectiveness and efficiency. 
While we do not foresee direct negative societal impacts arising from this work, we intend to further refine and extend the method in future research.

\newpage
\section*{NeurIPS Paper Checklist}

\begin{enumerate}

\item {\bf Claims}
    \item[] Question: Do the main claims made in the abstract and introduction accurately reflect the paper's contributions and scope?
    \item[] Answer: \answerYes{} 
    \item[] Justification: The main claims made in the abstract and introduction accurately reflect the paper's contributions and scope.

\item {\bf Limitations}
    \item[] Question: Does the paper discuss the limitations of the work performed by the authors?
    \item[] Answer: \answerYes{} 
    \item[] Justification: The limitations of this work are discussed in the last paragraph of the main text. 

\item {\bf Theory assumptions and proofs}
    \item[] Question: For each theoretical result, does the paper provide the full set of assumptions and a complete (and correct) proof?
    \item[] Answer: \answerYes{} 
    \item[] Justification: The full set of assumptions and a complete and correct proof are provided in the appendices. 

    \item {\bf Experimental result reproducibility}
    \item[] Question: Does the paper fully disclose all the information needed to reproduce the main experimental results of the paper to the extent that it affects the main claims and/or conclusions of the paper (regardless of whether the code and data are provided or not)?
    \item[] Answer: \answerYes{} 
    \item[] Justification: Detail experimental settings are provided in the appendices and the code to reproduce the main experimental results is released. 

\item {\bf Open access to data and code}
    \item[] Question: Does the paper provide open access to the data and code, with sufficient instructions to faithfully reproduce the main experimental results, as described in supplemental material?
    \item[] Answer: \answerYes{} 
    \item[] Justification: The code to faithfully reproduce the main experimental results is publicly released. 

\item {\bf Experimental setting/details}
    \item[] Question: Does the paper specify all the training and test details (e.g., data splits, hyperparameters, how they were chosen, type of optimizer) necessary to understand the results?
    \item[] Answer: \answerYes{} 
    \item[] Justification: All the training and test details necessary to understand the results are specified in appendices. 

\item {\bf Experiment statistical significance}
    \item[] Question: Does the paper report error bars suitably and correctly defined or other appropriate information about the statistical significance of the experiments?
    \item[] Answer: \answerYes{} 
    \item[] Justification: The mean and standard deviations in main results are reported from repeated experiments. 

\item {\bf Experiments compute resources}
    \item[] Question: For each experiment, does the paper provide sufficient information on the computer resources (type of compute workers, memory, time of execution) needed to reproduce the experiments?
    \item[] Answer: \answerYes{} 
    \item[] Justification: Information of the computer resources is provided. 
    
\item {\bf Code of ethics}
    \item[] Question: Does the research conducted in the paper conform, in every respect, with the NeurIPS Code of Ethics \url{https://neurips.cc/public/EthicsGuidelines}?
    \item[] Answer: \answerYes{} 
    \item[] Justification: The research conforms with the NeurIPS Code of Ethics. 

\item {\bf Broader impacts}
    \item[] Question: Does the paper discuss both potential positive societal impacts and negative societal impacts of the work performed?
    \item[] Answer: \answerYes{} 
    \item[] Justification: Societal impacts of this work are discussed in the end of the appendices. 
    
\item {\bf Safeguards}
    \item[] Question: Does the paper describe safeguards that have been put in place for responsible release of data or models that have a high risk for misuse (e.g., pre-trained language models, image generators, or scraped datasets)?
    \item[] Answer: \answerNA{} 
    \item[] Justification: The paper poses no such risks. 

\item {\bf Licenses for existing assets}
    \item[] Question: Are the creators or original owners of assets (e.g., code, data, models), used in the paper, properly credited and are the license and terms of use explicitly mentioned and properly respected?
    \item[] Answer: \answerYes{} 
    \item[] Justification: Assets used are properly credited. 

\item {\bf New assets}
    \item[] Question: Are new assets introduced in the paper well documented and is the documentation provided alongside the assets?
    \item[] Answer: \answerYes{} 
    \item[] Justification: The released code is well documented. 

\item {\bf Crowdsourcing and research with human subjects}
    \item[] Question: For crowdsourcing experiments and research with human subjects, does the paper include the full text of instructions given to participants and screenshots, if applicable, as well as details about compensation (if any)? 
    \item[] Answer: \answerNA{} 
    \item[] Justification: The paper does not involve crowdsourcing nor research with human subjects. 

\item {\bf Institutional review board (IRB) approvals or equivalent for research with human subjects}
    \item[] Question: Does the paper describe potential risks incurred by study participants, whether such risks were disclosed to the subjects, and whether Institutional Review Board (IRB) approvals (or an equivalent approval/review based on the requirements of your country or institution) were obtained?
    \item[] Answer: \answerNA{} 
    \item[] Justification: The paper does not involve crowdsourcing nor research with human subjects.

\item {\bf Declaration of LLM usage}
    \item[] Question: Does the paper describe the usage of LLMs if it is an important, original, or non-standard component of the core methods in this research? Note that if the LLM is used only for writing, editing, or formatting purposes and does \emph{not} impact the core methodology, scientific rigor, or originality of the research, declaration is not required.
    \item[] Answer: \answerNA{} 
    \item[] Justification: The core method development in this research does not involve LLMs as any important, original, or non-standard components. 

\end{enumerate}

\end{document}